\documentclass[letterpaper]{article} % DO NOT CHANGE THIS
\usepackage{aaai2026}  % DO NOT CHANGE THIS
\usepackage{times}  % DO NOT CHANGE THIS
\usepackage{helvet}  % DO NOT CHANGE THIS
\usepackage{courier}  % DO NOT CHANGE THIS
\usepackage[hyphens]{url}  % DO NOT CHANGE THIS
\usepackage{graphicx} % DO NOT CHANGE THIS
\usepackage{natbib}  % DO NOT CHANGE THIS AND DO NOT ADD ANY OPTIONS TO IT
\usepackage{caption} % DO NOT CHANGE THIS AND DO NOT ADD ANY OPTIONS TO IT
\usepackage{algorithm}
\usepackage{algorithmic}

\usepackage{newfloat}
\usepackage{listings}
\DeclareCaptionStyle{ruled}{labelfont=normalfont,labelsep=colon,strut=off} % DO NOT CHANGE THIS
\floatstyle{ruled}
\newfloat{listing}{tb}{lst}{}
\floatname{listing}{Listing}
\usepackage{mathtools}
\usepackage{subcaption}
\usepackage{our_pack}

\usepackage{hyperref} 
\nocopyright
\title{Exploiting Missing Data Remediation Strategies using Adversarial Missingness Attacks}
\author{
    Deniz Koyuncu\textsuperscript{\rm 1}\thanks{This work was done in part while the author was visiting Google LLC, New York, NY.}, Alex Gittens\textsuperscript{\rm 1}, Bülent Yener\textsuperscript{\rm 1}\footnotemark[1], Moti Yung\textsuperscript{\rm 2,3}
}
\affiliations{
    \textsuperscript{\rm 1}Rensselaer Polytechnic Institute, \textsuperscript{\rm 2}Google LLC,
    \textsuperscript{\rm 3}Columbia University\\
    
    {deniz\_kync@icloud.com}
}

\begin{document}

\maketitle

\begin{abstract}

Adversarial Missingness (AM) attacks aim to manipulate model fitting by carefully engineering a {\em missing} data problem to achieve a specific malicious objective.
AM attacks are significantly different from prior data poisoning attacks in that no malicious data inserted and no data is maliciously perturbed. Current AM attacks are feasible only under the assumption that the modeler (victim) uses full-information maximum likelihood methods to handle missingness. This work aims to remedy this limitation of AM attacks; in the approach taken here, the adversary achieves their goal by solving a bi-level optimization problem to engineer the adversarial missingness mechanism, where the lower level problem incorporates a differentiable approximation of the targeted missingness remediation technique. As instantiations of this framework, AM attacks are provided for three popular techniques: (i) complete case analysis, (ii) mean imputation, and (iii) regression-based imputation for general {\em empirical risk minimization} (ERM) problems. 
Experiments on real-world data show that AM attacks are successful with modest levels of missingness (less than 20\%). 
Furthermore, we show on the real-world {\em Twins} dataset that AM attacks can manipulate  the estimated average treatment effect (ATE) as an instance of the general ERM problems: the adversary succeeds in not only reversing the sign, but also in substantially inflating the ATE values from a true value of $-1.61\%$ to a manipulated one as high as $10\%$. These experimental results hold when the ATE is calculated using multiple regression-based estimators with different architectures, even when the adversary is restricted to modifying only a subset of the training data. The goals of this work are to: (i) establish the vulnerability to AM attacks of a significantly wider class of missingness remediation strategies than established in prior work, and (ii) brings the AM threat model to the attention of the community, as there are currently no defense strategies for these attacks. 
%Recently introduced Adversarial Missingness (AM) attacks have shown that mere malicious omission of data can be used to manipulate model fitting. However, existing AM attacks assume that the modeler (victim) uses full-information maximum likelihood methods to handle missingness, and are not applicable when the modeler uses different remediation strategies. This work presents a novel approach to AM attacks that works for general empirical risk minimization problems with a much wider class of remediation strategies by using probabilistic approximations based on the asymptotic forms of the strategies. This framework is applied to (i) complete case analysis, (ii) mean imputation, and (iii) regression-based imputation. The learning of the adversarial missingness mechanism is formulated as a bi-level optimization problem. Experiments show that these novel AM attacks can flip feature $p$-values from significant to insignificant on real datasets using relatively modest levels of missingness (less than 20\%). On the Twins dataset, they can also reverse the direction and substantially inflate the estimated average treatment effect (ATE)---from a true value of $-1.61\%$ to as high as $10\%$---when measured using multiple regression-based ATE estimators, even when the adversary is restricted to modifying only a subset of the training data. The AM induced bias remains even when the missing data is implicitly handled by the Causal Forest algorithm or is imputed using the popular MICE algorithm.
\end{abstract}

\begin{links}
    \link{Code}{https://github.com/cruyffturn/AM-AAAI26}
\end{links}

	\section{Introduction}
	
	%Due to the ubiquity of data sets with incomplete observations, there has been a recent resurgence in the development of ML methods for handling incomplete data; simultaneously, there has been considerable focus on data poisoning and adversarial attacks. However, the potential for missing data to serve as a vector for adversarial manipulation remains largely unexplored. Most prior works on data poisoning assume that the adversary can modify the training data set either by arbitrarily perturbing entries or inserting artificial data. These are powerful attacks, but adversarial missingness is an independent route of attack (one that attacks misplaced trust in the observed pattern of missingness), orthogonal to perturbative manipulations (which attack misplaced trust in the observed data), worth investigation on its own merits. This paper provides evidence that commonly used mechanisms for handling missingness in associative learning are indeed vulnerable to adversarial manipulation. 

Missing data is ubiquitous in real-world datasets, and recent machine learning work has renewed focus on handling it. Meanwhile, adversarial and data poisoning attacks have gained attention, but the notion that missingness itself could be adversarial remains largely unexplored. Existing poisoning models typically assume the attacker perturbs or injects data, whereas adversarial missingness (AM) harms by {\em selectively omitting data}--- a distinct and orthogonal attack vector. Standard defenses (e.g., data sanitization and outlier detection, statistically robust training) do not apply, as AM involves no modification of the observed data and can mimic naturally occurring missing-not-at-random (MNAR) patterns. 

%Missing data is ubiquitous in real-world datasets, and recent work in machine learning has renewed the focus on handling incomplete data. At the same time, adversarial and data poisoning attacks have garnered considerable attention. However, the possibility that missingness itself could be ``adversarial'' remains essentially underexplored. Existing poisoning models typically assume that the attacker perturbs or injects data, whereas adversarial missingness (AM) introduces harm by selectively omitting data— a distinct and orthogonal attack vector. Furthermore, defense techniques deployed against to data poisoning attacks (e.g., cryptographic signing, integrity checks) are not applicable to AM attacks since (i) they involve no data modification, and (ii) adversarial missingness can adhere to the missing not at random (MNAR) model, making it quite difficult to distinguish from benign MNAR missingness which is encountered naturally in real-world data. 

\begin{figure}[]
  	\centering
  	\includegraphics[width=0.4\textwidth]{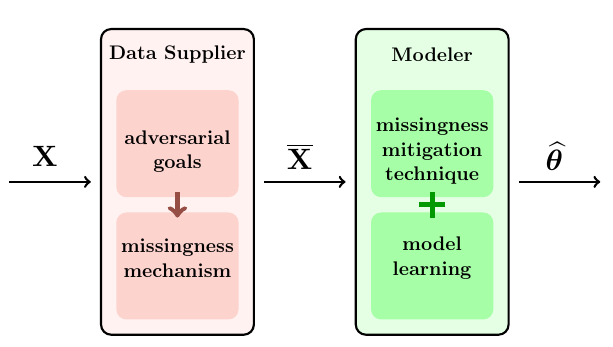}
  	\caption{The AM threat model: the adversary adds missingness to the initial dataset $\mZ$ to obtain $\bar{\mZ}$, and the modeler mitigates the missingness to learn a model $\that$. Because the adversary carefully engineered the missingness mechanism with a malicious goal in mind, $\that$ is steered to minimize the adversarial objective $g(\cdot; \mZ)$.}
  	\label{fig:model}
  	%	\label{fig:}
  \end{figure}
AM attacks, first explored in prior work~\cite{kdd, koyuncu2024adversarial}, model a setting where an adversary (as data supplier) omits a subset of the training data before it is processed by a modeler using a missingness remediation technique (e.g., imputation); see Figure~\ref{fig:model} for a visualization of the threat model. 
Existing AM attacks~\cite{kdd, koyuncu2024adversarial} have very narrow applicability, as they were introduced specifically to manipulate the learning of Gaussian Structural Causal Models when Full Information Maximum Likelihood (FIML) is used by the modeler as the missing data remediation strategy. FIML makes assumptions on the form of the joint distribution of the features, $X$, and the response, $Y$; this is appropriate when learning causal models, but is overly restrictive in most other learning settings. For example, to apply previous AM attacks to the setting of discriminative learning, one must assume that $Y$ and the features of $X$ are jointly Gaussian. Thus these attacks are not applicable when $Y$ is discrete (as in logistic regression); or when any of the features of $X$ are categorical, have multiple modes, are always positive, and so on.

	%Adversarial missingness (AM) attacks, introduced in \cite{kdd, koyuncu2024adversarial}, attack only by omitting a subset of the entries of a dataset. The AM threat model is conceptualized in terms of two actors: a data supplier, and a modeler. The adversary, acting as a data supplier, purposefully omits a subset of a dataset's entries. The victim (modeler), provided the partially observed dataset, applies a missing data remediation technique and learns a model. See Figure~\ref{fig:model} for an overview of the workflow of adversarial missingness (AM) attacks. The AM threat model presents an attractive prospect for data poisoning because missingness is ubiquitous and the current state of practice in ML does not include the consideration of adversarial introduction of missingness. AM attacks also bypass the previously listed examples of defenses against data poisoning: incompletely observed data can occur in the presence of crytographic signing, and does not rely on perturbations that may violate natural structures implicit in the data.

This paper introduces a general framework for AM attacks on empirical risk minimization (ERM) tasks with differentiable losses, without assuming specific distributions or model architectures. We instantiate this framework for widely-used remediation strategies: complete case analysis, mean imputation, and regression-based imputation, framing the attacker’s objective as a bi-level optimization problem using differentiable proxies. Figure~\ref{fig:example} shows an example of an AM attack on logistic regression with mean imputation developed using this framework--- a learning setting that cannot be addressed using  previous AM frameworks.

    %This paper introduces a framework for deriving a broader class of AM attacks applicable to {\em general empirical risk minimization} (ERM) tasks with differentiable loss functions, with none of the limitations of the FIML framework. These attacks assume no particular distributional form or model architecture. As instances of this framework, AM attacks are developed that target the widely-used missingness remediation strategies of {\em complete case analysis, mean imputation}, and {\em regression-based imputation}. Our method frames the attacker’s problem as a bi-level optimization and uses differentiable proxies to model the remediation strategies. Figure~\ref{fig:example} illustrates the result of an AM attack-- using the framework developed in this paper-- upon the learning of a logistic regression model with mean imputation; the methods of~\cite{kdd,koyuncu2024adversarial} cannot be used for even this simple toy example.

Our contributions are: (i) a general bi-level framework for AM attacks on differentiable ERM problems;
(ii) differentiable approximations of several popular missingness mitigation techniques;
(iii) empirical results showing successful manipulation of model behavior---including feature importance and treatment effects---using real-world data sets. Notably, our attacks often transfer across models and mitigation strategies.
% (i) a general bi-level optimization framework for AM attacks against differentiable ERM objectives; (ii) differentiable approximations of popular missingness mitigation methods; (iii) empirical results--including altering feature significance and estimated treatment effects-- illustrating the power of our novel AM framework to successfully manipulate model behavior in real-world datasets with model architectures, feature distributions, and missingness mitigation methods significantly more general than those that can be attacked with prior techniques. We observe that AM attacks designed for different architectures and missingness mitigation techniques often transfer to different architectures and mitigation techniques. 

%These results reveal a previously unrecognized vulnerability in ML pipelines that handle missing data using popular missingness mitigation methods and motivate further work on robust methods resistant to AM attacks.

Our empirical examples target tabular data sets, as missingness is natural in such data sets. We defer extensions to other modalities (e.g., deep nets, images) and efficiency improvements to future work. Our results expose a vulnerability in standard ML pipelines that handle missing data, motivating the need for defenses against AM attacks.

    \begin{figure}[]
  	\centering
  	\includegraphics[width=0.3\textwidth]{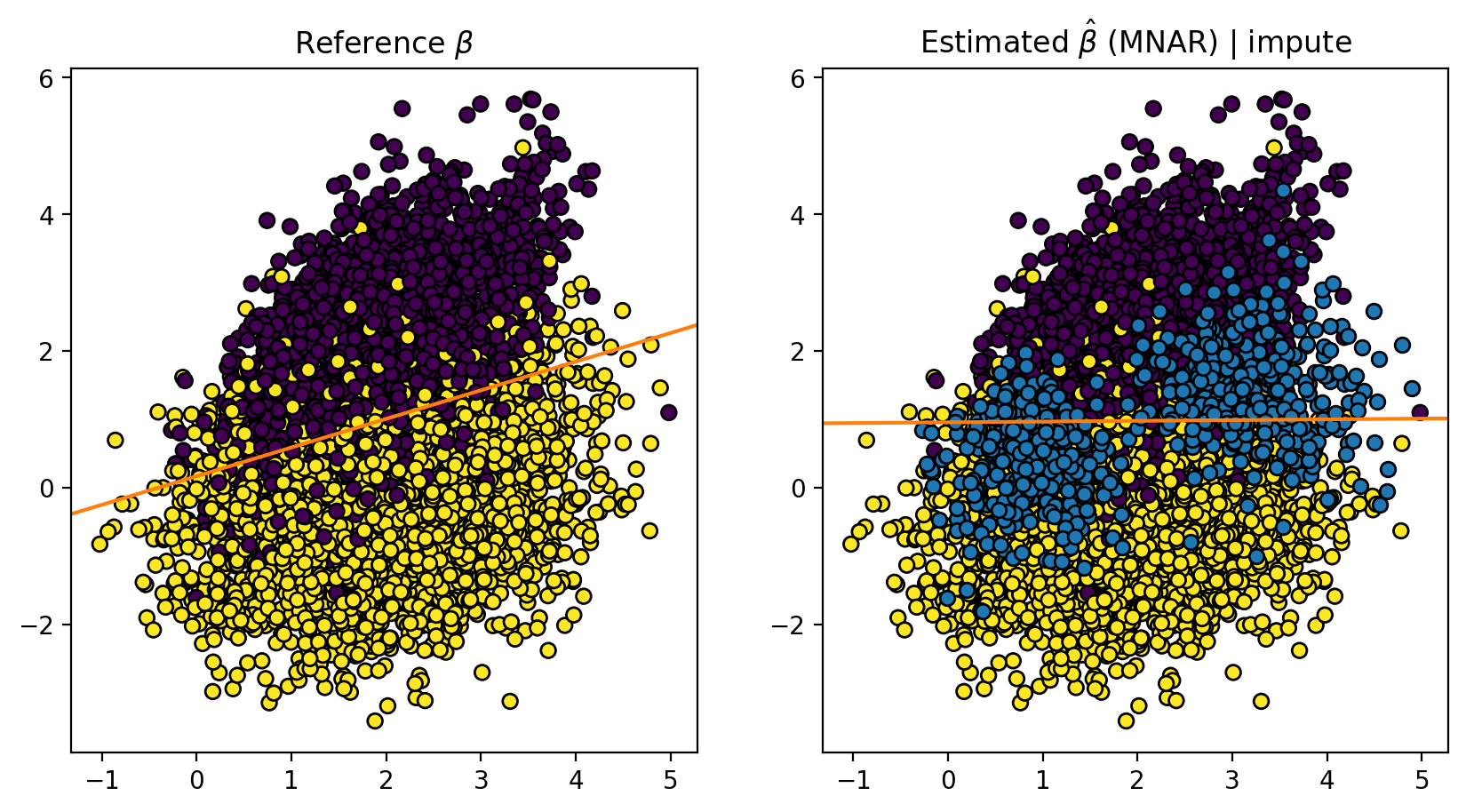}
  	\caption{Example of manipulating a logistic regression model for a classification problem. By omitting the x-coordinate of 8.4\% of the samples (colored blue), the \noun rotates the optimal decision boundary (left figure) to a horizontal line (right figure) under mean imputation. The classification accuracy decreased by 0.4\% but with high confidence the modeler asserts that the x variable has a coefficient close to zero (p value=0.688).}
  	\label{fig:example}
  	%	\label{fig:}
  \end{figure}

	\paragraph{Notation}
	The random vector corresponding to the training data is $\Z=(X_1,\dots,X_d)$, while a realization of that random vector (i.e. an instance of training data) is denoted by $\z=(x_1,\dots,x_d)$; the rows of the matrix $\mZ \in \R^{N \times d}$ constitute the training data $\{\z_i\}_{i=1}^N$. Similarly, the masking matrix $\mR \in \{0,1\}^{N \times d}$ introduced in the next section comprises rows of realizations $\rb_i$ of instances of the random vector $\Rb = (R_1, \ldots, R_d)$. Probability density functions (pdf) parameterized by a parameter $\vtheta$ are denoted by $\p(\z ; \vtheta)$ (which random variable is under consideration will be clear from context), and similarly, parameterized conditional pdfs are denoted using notation like $\p_{R\mid X}(\rb \mid \z; \vphi)$. 
	
	%[Abundance of MLx applications containing missing data]
	
	\section{Problem Formulation}
	%on FIML methods~\cite{koyuncu2024adversarial} assumes
	%The original formulation of AM attacks the adversary being the data owner and/or data distributor has access to all of the data modeler will use. 
	In AM attacks, the \noun is restricted to hiding existing entries.  We indicate the entries of $\mZ$ to be made missing with the binary masking matrix $\mR\in\{0,1\}^{N\times d}$. After selecting the masking matrix, the adversary hides the corresponding entries of the original matrix to obtain the matrix $\bar{\mZ}$ of partially observed data: $\bar{\mZ}_{i,j}=\text{NA}$ if $\mR_{i,j}=0$ while $\bar{\mZ}_{i,j}=\mZ_{i,j}$ if $\mR_{i,j}=1$. 
	
	The modeler uses the resulting partially observed dataset $\bar{\mZ}$ to learn a model by minimizing an objective function $f$ over a set of feasible parameters $\Theta$:
	\begin{equation}\label{eq:modeler}
		\that = \argmin_{\vtheta\in\Theta} f(\vtheta; \bar{\mZ})
	\end{equation}
	The modeler's objective $f$ accomplishes both the missingness mitigation and the learning of the model. For example, \cite{koyuncu2024adversarial} focuses on the case where the modeler fits a Gaussian generative model using the observed portion of the dataset, $\bar{\mZ}_\text{obs}$, and the objective is $f(\vtheta; \bar{\mZ})= -\log \p(\bar{\mZ}_\text{obs};\vtheta)$. In the more general learning setup considered in this paper, $f$ comprises the combination of a missingness mitigation technique (complete case analysis, mean imputation, or conditional mean imputation) with an arbitrary differentiable empirical risk minimization objective. 
	
	The model resulting from training, $\that$, depends on which entries were observed. The goal of the \noun is to steer $\that$ towards an adversarial model $\ta$ selected to achieve a specific adversarial objective, while introducing a minimal amount of missingness. The adversary may also select $\ta$ to limit the occurrence of auditable outcomes such as a low predictive performance or unintended differences between the models learned with and without adversarial manipulation. The central challenge in accomplishing the adversary's goal lies in effectively searching through the set of possible masking strategies, which has size $2^{N\times d}$.
	
	To bypass this combinatorial search problem, we replace the selection of the deterministic masking matrix $\mR$ with the selection of a \emph{missingness mechanism} $\p_{\Rb\mid \Z}$ that the adversary then uses to probabilistically generate missingness masks for the individual rows of $\mZ$. Given the $i$th observation, $\z^{(i)}$, the adversary samples the $i$th row of the masking matrix, $\rb^{(i)}$, proportional to $\p_{\Rb\mid \Z}(\cdot \,|\, \z^{(i)})$. We learn the adversarial missingness mechanism (AM mechanism) by solving a bi-level optimization problem in which the lower level problem corresponds to the modeler's objective and the upper level problem models the adversary's intent:
	\begin{equation}\label{eq:adv_obj}
\begin{split}				
	\min_{\p_{\Rb|\mZ}}\,  g(\ttil; \mZ) + &\lambda \cdot \Omega(\p_{\Rb|\Z};\mZ)\\
	&\text{s.t. } \quad \ttil = \argmin_{\vtheta \in \Theta} \tilde{f}(\vtheta, \p_{\Rb|\Z};\mZ)
\end{split}
	\end{equation}
    %\Delta(\ttil,\ta;\mZ)
	The parameter learned in the upper level is the AM mechanism $\p_{\Rb\mid \Z}$; given this AM mechanism, the lower level problem returns an approximation $\ttil$ to the model that the modeler would learn using the combination of their missingness mitigation and learning objective when given training data that was masked using $\p_{\Rb\mid \Z}$. 
	
	In the upper-level problem, the objective $g(\ttil; \mZ)$ captures the adversary's intent. For example, given a specific adversarial model $\ta$, the adversary may take $g(\ttil; \mZ) = \|\ttil - \ta\|_2^2$; if the goal were instead to cause a linear model to make specific predictions on each datapoint, an appropriate choice would be $g(\ttil; \mZ) = \|\mZ \ttil - \mathbf{y}_\text{target}\|_2^2$. The regularization term $\Omega$ ensures that the learned missingness mechanism has a low missingness rate. See Section~\ref{sec:solution} for the specific form of $\Omega$, and Section~\ref{sec:exp} for two examples of $g$ tailored respectively to reducing variable importance and manipulating average treatment effect estimation. Tuning $\lambda$ balances between achieving the adversarial objective and lowering the missingness rate.
	
	In general the modeler's objective $f$ is \emph{not} differentiable with respect to the missingness mechanism, so to facilitate bi-level learning with gradient methods, the lower-level problem simulates the modeler's training process using an approximation $\tilde{f}$ to the modeler's objective function $f$ that \emph{is} differentiable with respect to the missingness mechanism. 
	
	\cite{kdd, koyuncu2024adversarial} pioneered the use of missingness mechanisms to avoid the combinatorial search problem; these works developed AM attacks on Gaussian causal models learned using FIML. Our innovations are: (i) the introduction of a bilevel formulation for learning the missingness mechanism, as this allows the attacking of arbitrary ERM learning with differentiable objectives, and (ii) the development of effective differentiable proxies $\tilde{f}$ for common missingness remediation strategies used by modelers. 
	
	%More precisely, $\tilde{f}$ is constructed such that selecting $\that$ by maximizing it corresponds to the manner in which the modeler would learn a model that mini  $\vtheta$, given a population of partially observed datasets $\bar{\Z}$ resulting from the missingness mechanism $\P_{\Rb|\Z}$. The output $\ttil$ is referred to as the simulated parameters. As the overall goal is to learn an effective AM mechanism, $\tilde{f}$ is constructed to be differentiable with respect to $\P_{\Rb|\Z}$.
	
	%This formulation converts the combinatorial search to a learning problem. Its key component is the approximate objective $\tilde{f}$ which should  capture the modeler's behavior while remaining differentiable with respect to the missingness mechanism. 
	
	%In previous work on causal discovery \cite{kdd} the distance function measures the Kullback-Leibler divergence between the SCM corresponding to $\ttil$ and the target SCM with parameters $\ta$. 
	%\begin{figure*}
	%	\centering
	%	\includegraphics[width=1\linewidth]{img/example_09_05}
	%	\caption{In AM attacks, the adversary uses the samples in the training set to design a missingness mechanism $\P_{\Rb|\Z}$, given knowledge of the modeler's technique for mitigating missing data. The adversary then samples a masking matrix from the missingness mechanism, replaces the indicated entries with $\text{NaN}$s, and conveys this poisoned dataset to the modeler. The modeler applies CCA or imputation to the partially observed dataset, then proceeds to learn the model.}
	%	\label{fig:overview}
	%\end{figure*}
	
	\section{Differentiable Proxy Objectives for Missing Data Remediation}\label{sec:derive}

    In this section we provide one of the major contribution of this work: novel differentiable proxies $\tilde{f}$ for complete-case analysis (CCA) and mean imputation; using similar ideas, a proxy for conditional mean imputation (using linear regression) is developed in Appendix~\ref{sec:ls}. These proxies allow the adversary to use (\eqref{eq:adv_obj}) to learn adversarial missingness mechanisms targeting these missingness mitigation techniques.
    
	To design differentiable approximations to the modeler's objective, we assume that the modeler's goal, given a completely observed data set $\mZ$, is to solve the empirical risk minimization problem 
	\begin{equation}\label{eq:modelerfulldata}
		\that = \argmin_{\vtheta\in\Theta} \frac{1}{N} \sum\limits_{i=1}^N J (\z^{(i)};\vtheta)
	\end{equation} 
	for a specified score function. For instance, if the modeler's goal is to fit a logistic regression model on data $\z^{i} = (\bm{\omega}_i, y_i)$, then $J(\z^{i}; \vtheta) = -\log(1 + \exp(- y_i \vtheta^T \bm{\omega}_i)).$
	
	From here on, to simplify the presentation we assume that the missingness mechanism $\p_{\Rb|\Z}$ is completely determined by a parameter vector $\vphi$, so that ``differentiability with respect to the missingness mechanism" means differentiability with respect to $\vphi$ and learning $\p_{\Rb|\Z}$ reduces to a finite-dimensional optimization problem.

	\subsection{Complete-Case Analysis}\label{sec:cca}
	
	In CCA~\cite{littleStatisticalAnalysisMissing2002}, the modeler discards all rows with missing entries before proceeding with their analysis. Denoting the set of completely observed rows by ${\mathcal{S}=\{i:\rb^{(i)}=\mathbf{1}\}}$, the modeler's objective (\eqref{eq:modeler}) is to learn $\that$ to minimize the average score of the completely observed examples:
	\begin{equation}\label{eq:modeler_cca_emp}
		f(\vtheta; \bar{\mZ})=\frac{1}{|\mathcal{S}|} \sum\limits_{i\in \mathcal{S}} J (\z^{(i)};\vtheta)
	\end{equation}
	The random set $\mathcal{S}$ that the modeler has to work with is determined by the adversary's missingness mechanism. The missing data masks are {\em sampled} from the AM mechanism, so the objective function (\eqref{eq:modeler_cca_emp}) is not differentiable with respect to $\vphi$, and cannot be used in the bi-level formulation used to learn the AM mechanism.
	
	However, because the training samples are i.i.d., as the number of samples goes to infinity, the weak law of large numbers asserts that, for any $\vtheta$, the modeler's objective (\eqref{eq:modeler_cca_emp}) converges in probability to the expected score of $\vtheta$ conditioned on all entries in $X$ being observed. 
	\begin{align*}
		f(\vtheta; \bar{\mZ}) & \overset{P} {\rightarrow} \: \E[J(\Z;\vtheta)\mid \Rb=\mathbf{1}] \stepcounter{equation}\tag{\theequation}\label{eq:modeler_cca}\\
	%	&= \int J(\z;\vtheta) \p_{\Z \mid \Rb}(\z \mid \mathbf{1};\vphi)  \; d\z \\
		&=  \frac{1}{\P_{\Rb;\vphi}(\mathbf{1})} \int J(\Z;\vtheta) \p_{\Rb \mid \Z}(\mathbf{1} \mid \z;\vphi) \p_X(\z) \; d\z.
	\end{align*}
	The equality follows from an application of Bayes' Theorem.
	%When $\P_{\Rb}(\mathbf{1})>0$, Bayes' Theorem can be applied to capture the exact asymptotic dependence of the modeler's objective on the missingness mechanism:
	
	Equation \ref{eq:modeler_cca} clarifies that, under CCA, the asymptotic impact of the missingness mechanism is to weigh the score function by the missingness mechanism. Larger weights are given to observations that are likely to be completely observed under the missingness mechanism. The asymptotic objective in (\eqref{eq:modeler_cca}) is differentiable with respect to the missingness mechanism, but cannot be evaluated on finite training data. 
	
	To obtain the final proxy objective that is both differentiable with respect to the missingness mechanism and can be evaluated on finite training data, we take $\tilde{f}$ to be the approximation of the expectation in (\eqref{eq:modeler_cca}) on the training data:
	\begin{equation*}
		\tilde{f}(\vtheta, \vphi; \mZ)=
		\frac{1}{\P_{\Rb;\vphi}(\mathbf{1})}\frac{1}{N}\sum\limits_{i=1}^{N}\p_{\Rb\mid \Z}(\mathbf{1}\mid \z^{(i)};\vphi)    J(\z^{(i)};\vtheta).\tag{5.1}\label{eq:cca_weight}
	\end{equation*}
	Directly choosing an adversarial mask matrix $\mR$ to attack a CCA modeler would require searching over the $2^{N}$ possible subsets of observed rows; while the bi-level probabilistic formulation of AM attacks on CCA modelers using this proxy objective involves optimizing over an inner objective function that is computable in time linear in the number of training observations.
	
	\subsection{Missing Data Imputation}
	
	When the modeler uses imputation, the missing entries in data matrix $\bar{\mZ}$ are imputed to obtain a completed matrix, which we denote by $\hat{\mZ}$. As a popular example, consider mean imputation; here, the missing entries in the $j$th column are replaced with the average of the observed entries in that column, denoted by $\hat{\vmu}_j\in\real$. Consequently, $\hat{\mZ}_{i,j}=\hat{\vmu}_j$ if $\rb_{i,j}=0$ while $\hat{\mZ}_{i,j}=\mZ_{i,j}$ if $\rb_{i,j}=1$. After the imputation, the modeler uses the resulting complete dataset to train a model, so the objective of a modeler using mean imputation is given by
	\begin{equation}\label{eq:imp_emp}
		f(\vtheta; \bar{\mZ})= \frac{1}{N}\sum_{i=1}^{N} J (\zh^{(i)};\vtheta). 
	\end{equation}
	Unlike CCA, in general, imputation introduces dependence between the initially independent observations. Consequently, the weak law of large numbers cannot be readily used to obtain a differentiable proxy function for (\eqref{eq:imp_emp}). 
	
	To find such a proxy function, we initially assume the imputation model is not learned from data, but instead is fixed \emph{a priori}. In this case, the rows of $\hat{\mZ}$ are i.i.d., and a similar argument to before gives the asymptotic behavior
	\begin{align*}
		f(\vtheta; \bar{\mZ}) &\overset{P}{\rightarrow} \:  \E_{\Z,\Rb,\Zh}[J(\Zh; \vtheta)] \stepcounter{equation}\tag{\theequation}\label{eq:obj}\\
		%=& \E_{\Z}\left[ \E_{\Rb\mid\Z}\left[ \E_{\Zh \mid \Z, \Rb} [J(\Zh; \vtheta)\right]\right]\\
		 =& \E_{\Z}\left[ \sum_{\rb}\p_{\Rb\mid \Z}(\rb \mid \Z;\vphi) \E_{\Zh\mid \Z,\Rb=\rb}[J(\Zh; \vtheta)]\right],
	\end{align*}
	where $\rb \in \{0,1\}^d $ varies over the missingness masks that have nonzero probability under the missingness mechanism $\p_{\Rb \mid \Z}$, and the conditional pdf $\p_{\Zh\mid \Z,\Rb}$ denotes the fixed imputation model that takes in a row of incompletely observed entries and imputes the missing entries. This gives us a differentiable approximation with respect to $\vphi$. To evaluate it with finite-data, we empirically approximate the two remaining expectations:
	\begin{equation}\label{eq:imp_simple}\
		\tilde{f}(\vtheta, \vphi; \mZ)=\frac{1}{N}\sum\limits_{i=1}^{N}\sum\limits_{\rb\neq \mathbf{0}} \p_{\Rb\mid \Z}(\rb \mid \z^{(i)};\vphi) J(\ztphi; \vtheta),
	\end{equation}
	where $\ztphi$ is sampled\footnote{For simplicity, a single sample is used to approximate the inner-most expectation; more could be employed.} proportionally to the fixed imputation model $\p_{\Zh\mid \Z,\Rb}(\cdot\mid\z^{(i)},\rb)$
	
	In the worst case, when the missingness mechanism allows masking each $d$ feature of $X$, evaluating this proxy objective requires $N2^{d}$ summations; this is a significant reduction in complexity compared to directly looking for an adversarial mask, which searches over a search space  of size $2^{N\times d}$. However, we reached~(\eqref{eq:imp_simple}) by assuming that the imputation mechanism is fixed before seeing the data $\overline{\mZ}$. In practice, the imputation depends on the observed data, and thus on the missingness mechanism. 
	
	To capture that dependence, we propose to use the asymptotic forms of the imputation methods to derive expressions for $\ztphi$ that are differentiable with respect to the missingness mechanism; then we use~(\eqref{eq:imp_simple}) with these $\ztphi$. In the subsequent part, we derive the asymptotic form of the commonly used mean imputation. The asymptotic form of linear regression-based imputation is provided in the Appendix~\ref{sec:ls}.
	
	\paragraph{Mean Imputation} By the weak law of large numbers, the imputed value of an unobserved entry in the $j$th column converges to the conditional mean of the $j$th variable given that this variable is observed, i.e. $\E[\Z_j \mid \Rb_j = 1]$. This conditional expectation can be expressed in terms of the missingness mechanism as follows:
	\begin{equation}\label{eq:impute}
		\E[\Z_j \!\mid\! \Rb_j = 1]=\frac{1}{\P_{\Rb_j;\vphi}(1)}\E\left[ \Z_j \!\!\!\sum_{\forall \rb, \rb_j=1} \!\!\!\!\p_{\Rb\mid \Z}(\rb \mid \Z;\vphi)\right].
	\end{equation}
	See Proposition~\ref{prop:mean} in the Appendix for a proof. 
	
	The resulting imputed vector of the $i$th sample with missingness mask $\rb$ is denoted by $\ztphi\in \real^d$, and its elements satisfy: 
		\begin{equation}\label{eq:tilde_xh}
		\ztphi_j= \rb_j \z_j^{(i)} + (1-\rb_j)\hat{\vmu}_j(\vphi).
	\end{equation}
	Here, $\hat{\vmu}_j(\vphi)$ denotes a finite data approximation of the conditional expectation in (\eqref{eq:impute}). To compute $\hat{\vmu}_j(\vphi)$, first we empirically approximate the marginal probability of observing the $j$th feature under the AM mechanism as
	\begin{equation}
		\pi_j(\vphi)\defeq \frac{1}{N} \sum_{i=1}^{N}\sum_{\forall \rb, \rb_j=1}\p_{\Rb\mid \Z}(\rb \mid \z^{(i)};\vphi).
	\end{equation}
	Next, we empirically approximate the expectation in (\eqref{eq:impute}) as
	\[\hat{\vmu}_j(\vphi) \defeq \frac{1}{N \pi_j(\vphi)}\sum_{i=1}^{N} \z_j^{(i)}\sum_{\forall \rb, \rb_j=1}\p_{\Rb\mid \Z}(\rb \mid \z^{(i)};\vphi).\]
	As $\ztphi$ is now differentiable with respect to $\vphi$, using a sample from (\eqref{eq:tilde_xh}) in the proxy objective $\tilde{f}(\vtheta, \vphi;\mZ)$ from (\eqref{eq:imp_simple}) gives a proxy objective function for modelers using mean imputation that is differentiable with respect to $\vphi$.

	\section{Solving the Bi-level Problem}
	\label{sec:solution}
	The previous section derived differentiable formulations of the inner problems in our bi-level formulation (equation~\ref{eq:adv_obj}) of AM attacks, when the modeler uses CCA, mean imputation, and regression-based imputation. In this section, we discuss the solution of the bi-level problem. For notational brevity we drop the dependence on the dataset $\mZ$ throughout this section. We take $\Omega(\phi;\mZ)$ to be the empirical approximation of the expected fraction of missing data, so that the upper level objective is
	\begin{equation*}
\ell(\vphi,\vtheta) \defeq g(\ttil; \mZ) + \frac{\lambda}{N} \sum_{i=1}^N \E \bigg[\left. \frac{|\{j\,|\,\Rb_j=0\}|}{d}\right\vert \Z =\z^{(i)}\bigg].
	\end{equation*}
	%denote the approximate objective as $\tilde{f}(\vtheta, \vphi)$ by 
	%\begin{align*}
	%	\ell(\vphi, \bb)= &\frac{1}{N}\sum_{i=1}^N{\text{D}_{\textrm{KL}}(\P_{Y\mid \X^{(i)};\ba}\,\|\, \P_{Y\mid \X^{(i)};\bb})} + \\
	%	&\frac{\lambda}{N} \sum_{i=1}^N \E_{\Rb|\Z;\vphi} \left[\frac{|\{j\,|\,\Rb_j=0\}|}{d}\mid \Z =\z^{(i)}\right]
	%\end{align*}
	The corresponding bi-level optimization problem is
	\begin{align}\label{eq:lamm}
		\min_{\vphi} \ell(\vphi, \ttil(\vphi)), \ \text{s.t.} \quad \ttil(\vphi) = \argmax_{\vtheta\in \Theta} \ft(\vtheta,\vphi)
	\end{align}
	We use gradient descent on the upper level problem (the adversary's objective) to learn $\vphi$, the parameters of the missingness mechanism; computation of the gradient with respect to $\vphi$ using the implicit function theorem is standard, and is described in Appendix~\ref{sec:blamm}. We call the resulting algorithm the Bi-level Formulation for Learning AM Mechanisms (BLAMM). A listing is provided as Algorithm~\ref{alg:blamm} in Appendix~\ref{sec:blamm}. Solvers that are more scalable \cite{zhangIntroductionBilevelOptimization2023} or more suitable to complex optimization problems such as deep learning \cite{shenSEALSafetyenhancedAligned2024,haoBilevelCoresetSelection2023} exist and can replace the exact solver we used.

    Computation of the gradient requires solving the lower-level problem at each iteration. In Section \ref{sec:exp}, the runtime of the iterated least squares solver we used for the lower-level problems scales linearly with the number of samples in the dataset. In datasets that do not fit into memory, stochastic lower-level solvers can be used. The lower-level problems use objectives that sum over all possible missingness masks, incurring exponential growth with the number of masked variables; in practice, we found masking at most two carefully selected features was sufficient. 
	
We consider two methods of parameterizing the missingness mechanism, $\p_{\Rb|\Z}(\rb|\z,\vphi)$. The first, proposed in \cite{koyuncu2024adversarial}, uses a neural network to determine the probability of observing the masked features $\mask$ for a given instance $\z$. Specifically, the neural network outputs a probability distribution over the $2^{|\mask|}$ feasible missingness patterns-- those patterns $\rb$ in which all non-masked features and a subset of the masked features are observed. 
    
    The second method parameterizes the missingness mechanism on a per data-point basis: a distinct parameter vector $\vphi^{(i)} \in \mathbb{R}^{2^{|\mathcal{M}|}}$ is learned for each data point $\z^{(i)}$ and the softmax function is applied to $\vphi^{(i)}$ to obtain the probabilities of each feasible missingness pattern specific to that data point. Further details are provided in Appendix~\ref{sec:nnparam}.
	
	\section{Experiments}
	\label{sec:exp}
	%Our general derivation does not assume a specific modeler score function. As a proof of concept, we designed AM attacks for manipulating the learning of linear regression and logistic regression models. 
	%fitting a Generalized Linear Model (GLM) to the conditional distribution of a response random variable $Y$, given the vector of features $\X$. Following Almudevar~\cite{almudevar2021theory}, GLMs in canonical form are parameterized using a linear predictor term $\boldsymbol{\eta} = \inp{\bb}{\x}$ and have probability density (or mass) function:
	%\begin{equation}\label{eq:glm}
	%	\begin{split}
	%		p(y\mid\x; \bb)
	%		&= h(y, \sigma) \exp\left(\frac{\boldsymbol{\eta} y - A(\eta)}{\sigma}\right)
	%	\end{split}.
	%\end{equation}
	%In the following, the dispersion parameter is constant unless otherwise stated, and is therefore omitted from the notation. In our notation,  $\Z=(\X,Y)$ and $\vtheta=\bb$. GLMs are generally estimated using maximum likelihood estimation which is equivalent to using the log-likelihood function as the score function, i.e. $J(\z;\theta)=\log p(y\mid\x; \bb)$. We provided the modeler's corresponding differentiable objectives under CCA and imputation are in Appendix \ref{xx}. Due to the fact that MLE of GLMs is a convex-problem, selecting the missingness mechanism through the bi-level problem is simplified.

        Our experimental study focuses on tabular data which is of interest for many applications (medical EHR, product reviews, insurance and census data, etc.). 
        
        The effectiveness of  AM attacks are evaluated on two tasks: (i) manipulating the p-values of feature coefficients in linear and logistic regression models, and (ii) manipulating average treatment effect estimation using regression estimators. The attacks are successful even when: (i) the information available to adversary is limited (e.g. the missingness mitigation technique used by the modeler is unknown), and (ii) the percentage of the training data that can be modified is limited.
	
	\begin{table*}[ht!]
		\centering
		\caption{The average (over 20 trials) normalized $\ell_1$ norm of the difference between the modeler-estimated coefficients and the adversarial coefficients, i.e.
			$||\hat{\bm{\theta}}-\bm{\theta}_\alpha||_1/||\bm{\theta}_\alpha||_1$
%			 $\tfrac{||\hat{\bm{\theta}}-\bm{\theta}_\alpha||_1}{||\bm{\theta}_\alpha||_1}$
			  and its standard deviation (denoted as $\pm$). If the target coefficient resulting from the attack is insignificant on average (i.e. average $p$-value $>$ 0.05, Table~\ref{tab:pvalue} in the Appendix \ref{app:results}), this is indicated using a $\checkmark$.}
		\label{tab:dist_a}
		% \resizebox{0.75\textwidth}{!}{%
		% 	\input{table_a.tex}}
        
			% Table generated by Excel2LaTeX from sheet 'dist_a_short'
\begin{tabular}{l|l|ll|ll|ll}
\multicolumn{2}{c|}{Modeler/Attacker} & mean/mean & \multicolumn{1}{l}{mean/cca} & cca/mean & \multicolumn{1}{l}{cca/cca} & mice/mean & mice/cca \bigstrut[b]\\
\hline
\hline
ca-   & BLAMM  & \textbf{0.01$\pm$0.0 (\checkmark)} & \textbf{0.01$\pm$0.0 (\checkmark)} & \textbf{1.12$\pm$0.0} & \textbf{0.06$\pm$0.0 (\checkmark)} & \textbf{0.69$\pm$0.0} & \textbf{0.27$\pm$0.0} \bigstrut[t]\\
housing & MCAR  & 0.76$\pm$0.0 & 0.46$\pm$0.0 & 1.16$\pm$0.0 & 1.17$\pm$0.0 & 1.16$\pm$0.0 & 1.16$\pm$0.0 \bigstrut[b]\\
\hline
wine- & BLAMM  & \textbf{0.01$\pm$0.0 (\checkmark)} & \textbf{0.04$\pm$0.0 (\checkmark)} & \textbf{0.71$\pm$0.0} & \textbf{0.04$\pm$0.0 (\checkmark)} & \textbf{0.54$\pm$0.0} & \textbf{0.20$\pm$0.0} \bigstrut[t]\\
quality & MCAR  & 0.69$\pm$0.0 & 0.47$\pm$0.0 & 0.87$\pm$0.0 & 0.86$\pm$0.0 & 0.78$\pm$0.0 & 0.68$\pm$0.0 \bigstrut[b]\\
\hline
german- & BLAMM  & \textbf{0.01$\pm$0.0 (\checkmark)} & \textbf{0.01$\pm$0.0 (\checkmark)} & 0.38$\pm$0.0 (\checkmark) & \textbf{0.09$\pm$0.0 (\checkmark)} & \textbf{0.03$\pm$0.0 (\checkmark)} & \textbf{0.01$\pm$0.0 (\checkmark)} \bigstrut[t]\\
credit & MCAR  & 0.10$\pm$0.0 & 0.10$\pm$0.0 & \textbf{0.16$\pm$0.0} & 0.24$\pm$0.1 & 0.10$\pm$0.0 & 0.11$\pm$0.0 \\
\end{tabular}%

	\end{table*}%

	\subsection{Manipulating p-values of Features}\label{sec:attack1}

    Linear and logistic regression models are the most popular instance of generalized linear models (GLMs). GLMs are learned by maximizing appropriate log-likelihood functions and are widely used in data analysis, so their vulnerabilities to AM attacks has real-world implications. In our notation, the negative log-likelihood function of the GLM corresponds to the score function $J(\z;\vtheta)$ where $\vtheta$ are the coefficients of the model and $\z$ contains both the features and the response variable. %We provide the corresponding differentiable objectives used in our bi-level optimization framework for the GLMs in Appendix~\ref{sec:gradient}. 

	In our experiments the adversary aims to make the modeler statistically confident that the coefficient of a target variable is zero, $\vtheta_t=0$. The AM attack is deemed successful if the average $p$-value of the target coefficient is greater than $0.05$, indicating that the modeler fails to reject the null hypothesis that $\vtheta_t = 0$. To minimize the change in the predictive accuracy under the AM attack, the remaining coefficients of the adversarial target $\ta$ are selected by finding the closest GLM to the underlying data, subject to the condition $ \vtheta_{\alpha,t} = 0$; that is, they are determined using constrained maximum likelihood estimation (see Appendix~\ref{sec:klmin} for details). %To verify that the adversary shifts the learned parameters from the partially observed dataset $\that$ toward the adversarial $\ta$, we report the average normalized $\ell_1$-norm of their difference.
	
	We used two classification (wine-quality, german-credit) and two regression (ca-housing, diabetes) datasets (see Table \ref{tab:dataset} in the Appendix) to test our attacks. In each dataset, we selected one highly statistically significant feature, as identified using a GLM learned on the complete data, as the target coefficient. In all experiments, the AM mechanism is parameterized using a one-hidden layer neural network with 100 neurons in the hidden layer, and the masking set is restricted to the target feature, $\mask=\{t\}$. In training the AM mechanism, the adversarial objective $g$ is taken to be the empirical approximation to the KL-divergence between the learned GLM and the adversarial GLM (see Appendix~\ref{sec:klmin} for its formulation). %In all datasets except diabetes we trained three models corresponding to the CCA, mean imputation, and the regression-based attack. 
	
	The resulting missingness rates of the target variable ranged from 4.5\%-18.1\% in the four datasets, except for the CCA attack on the ca-housing dataset (see Table~\ref{tab:rate} in the Appendix). On this dataset, BLAMM for the CCA attack converged to masking 40.2\% entries of the target feature.
	
	To compare the learned AM mechanism with a reasonable baseline, we defined a missing completely at random (MCAR) missingness mechanism with (asymptotically) the same amount of missing data in the same features, given by ${\p_{\Rb}(\rb)=N^{-1}\sum_{i=1}^{N}\P_{\Rb \mid \Z}(\rb \mid \z^{(i)};\vphi)}$. We sampled 20 masking matrices $\mR$ from both the learned and the MCAR missingness mechanism. Given the partially observed dataset, the modeler first applied either one of CCA, mean imputation, or the MICE algorithm (r-package ``mice'' v 3.14.0~\cite{buurenMiceMultivariateImputation2011a}) to fill in the missing data. See Appendix~\ref{sec:modeler} for details. Next, using the resulting data set, the modeler estimated the coefficients of the models and their corresponding $p$-values. Additional experiments regarding the tailored attack for linear regression imputation are provided in the Appendix \ref{app:results_lr}.
	
	When the attack type matched the modeler's type, we observed that the learned adversary in all cases successfully made the target variable insignificant (see Table~\ref{tab:dist_a}, Table~\ref{tab:linear} in the Appendix \ref{app:results} for regression-based imputation attacks, and Table~\ref{tab:norm_diabetes} in the Appendix \ref{app:results} for diabetes dataset). In a stark contrast, MCAR was unsuccessful in \emph{all} cases. When there was a mismatch between the attack and modeler type, mean imputation attacks showed limited generalization while CCA generally successfully manipulated the coefficients of the modeler using mean imputation. The \mice algorithm was the most robust imputation strategy but failed in preventing the target variables being insignificant against the CCA attacks in the German-credit and $\text{diabetes}$ datasets.
%	and outperformed mean imputation attacks on under MICE imputation.
	
%	 regression-based imputation. 
	
%The \mice algorithm was the most robust but still failed in preventing the target variables being insignificant against the CCA and regression-based imputation attacks in the $\text{diabetes}$ and German-credit datasets.
	
	\paragraph{Data Valuation as a Defense}
    Data valuation methods assign utility scores to training examples using a clean validation set. A common defense against data poisoning \cite{just2023lava} discards the lowest-utility examples, up to a preset budget, before training. Since no existing defense methods exist for the AM threat model, we tested this strategy against our AM attacks by discarding samples after imputation. Due to space constraints, the results are provided in Appendix~\ref{sec:datavaluation}.
    
	%Data valuation methods assign a data utility score to each training example by contrasting them against a clean validation set. The following is a typical defense strategy that uses data valuation to defend against data poisoning \cite{just2023lava} : given a budget of samples that may be discarded, and the utility scores of the training examples, these methods discard the examples with the smallest utility scores (as many until the budget is met) before training the model with the remaining examples. Because to our knowledge there are no data valuation methods for incompletely observed data, we evaluated the efficacy of these defenses against AM attacks by discarding samples after imputation. Due to space constraints, the results are provided in Appendix~\ref{sec:datavaluation}.
			\begin{table*}[t!]
	\centering
	\caption{BLAMM attacks trained with a linear proxy are effective against non-linear models. The average ATE (\%) of regression estimators, $\hat{\tau}$ (over 5 trials). ``Access:'' indicates the percentage of rows manipulatable. Bold highlights the missingness closer to the target ATE of 10\%. See Appendix Tables \ref{app:tab:ate_mice}, \ref{app:tab:mnar},\ref{app:tab:mnar2} for the additional baselines described in the text.}
	% \resizebox{1\textwidth}{!}{\input{table_ate_n_5}}
\begin{tabular}{l|cl|ll|ll|ll}
\multicolumn{1}{r}{} & \multicolumn{2}{c}{Access: 100\%} & \multicolumn{2}{c}{Access: 75\%} & \multicolumn{2}{c}{Access: 50\%} & \multicolumn{2}{c}{Access: 25\%} \\
\multicolumn{1}{r}{} & \multicolumn{1}{l}{BLAMM} & \multicolumn{1}{l}{MCAR} & BLAMM & \multicolumn{1}{l}{MCAR} & BLAMM & \multicolumn{1}{l}{MCAR} & BLAMM & MCAR \bigstrut[b]\\
\hline
\hline
TARnet + mean & \textbf{10.52$\pm$0.9} & -1.0$\pm$2.2 & \textbf{8.49$\pm$1.7} & -0.38$\pm$1.9 & \textbf{7.26$\pm$2.3} & 0.04$\pm$0.8 & \textbf{1.62$\pm$2.0} & -0.42$\pm$1.8 \bigstrut[t]\\
Tnet  + mean & \textbf{10.42$\pm$0.9} & -2.6$\pm$3.1 & \textbf{7.44$\pm$1.4} & -1.72$\pm$3.2 & \textbf{3.67$\pm$1.4} & -4.7$\pm$1.2 & \textbf{-1.5$\pm$0.1} & -3.98$\pm$0.9 \\
linear + mean & \textbf{9.86$\pm$0.1} & -1.45$\pm$0.2 & \textbf{10.1$\pm$0.0} & -1.56$\pm$0.2 & \textbf{6.29$\pm$0.0} & -1.65$\pm$0.1 & \textbf{2.26$\pm$0.0} & -1.33$\pm$0.2 \\
CF + mean  & \textbf{7.65$\pm$0.5} & -1.44$\pm$0.2 & \textbf{3.32$\pm$0.1} & -1.45$\pm$0.3 & \textbf{3.01$\pm$0.0} & -1.3$\pm$0.1 & \textbf{1.54$\pm$0.0} & -1.25$\pm$0.1 \\
\end{tabular}%
%
   % your table here}
	\label{tab:ate}%
\end{table*}%
\subsection{Manipulating ATE under Partial Data Access}\label{sec:ate}
	As an additional demonstration of the potential of AM attacks we explore their efficacy in manipulating the estimation of average treatment effects (ATEs). The ATE quantifies the causal effect of changing a treatment variable $W$ on an outcome variable $Y$. When $W$ is binary-valued, the ATE measures the expected difference in the outcome when $W$ is set to 1 versus when $W$ is set to 0. Formally, using the do-calculus~\cite{pearlCausalityModelsReasoning2000}, the ATE is expressed as $\tau=\E[Y\mid \text{do}(W=1)]-\E[Y\mid \text{do}(W=0)]$. 
    For example, if $W$ indicates whether a person received a flu shot and $Y$ indicates whether they caught the flu, the ATE measures the expected change in a person's chance of getting the flu if they were vaccinated compared to if they were unvaccinated.

	%Estimating the ATE from observational data is challenging because, in general it is the case that $\E[Y\mid \text{do}(W=w)] \neq \E[Y \mid W=w]$. 
    Various methods have been proposed to address the challenge of ATE estimation. In this section we focus on the popular class of \emph{regression estimators}, which rely on learning the function $\mu_w(x)=\E[Y\mid X=x, \text{ do}(W=w)]$ under assumptions such as \emph{unconfoundedness} and \emph{overlap}~\cite{imbensNonparametricEstimationAverage2004}. When these assumptions hold, $\mu_w(x)$ can be estimated using the conditional expectation $\E[Y \mid X=x, W = w]$. The ATE is then estimated by computing the average difference of the predicted outcomes, $\hat{\tau} = \E[\hat{\mu}_1(x)-\hat{\mu}_0(x)].$ 
	
%	Given a set of covariates $\rvx$, Two common assumptions are unconfoundedness and overlap which allows that satisfy the , using the conditional expected value of the outcome $\mu_w(x) =E[Y\mid \rvx, W=w]$ given the covariates and $\{W=w\}$, i.e. $E[Y\mid \text{do}(W=w)]=E[\mu_w(X)]$ and averaging over the dataset 
%	 expected outcome under an intervention by computing the expectation of the  \cite{imbensNonparametricEstimationAverage2004}. 
%	Regression estimators \cite{imbensNonparametricEstimationAverage2004} use an unbiased estimator of the inner expectation to 
	
Simple regression estimators use linear models for $\hat{\mu}_w(x)$ by treating $W$ as an additional covariate \cite{imbensNonparametricEstimationAverage2004}. More recently, neural network models have been proposed. For instance, T-Net trains separate MLPs for each treatment group using the data subsets where $W=w$~\cite{curthNonparametricEstimationHeterogeneous2021}. TARNet improves upon this by learning a joint representation layer trained using all samples which is used as input to treatment level-specific hypothesis layers trained on the corresponding subsets of the data~\cite{shalitEstimatingIndividualTreatment2017a}. Other nonparametric methods include Causal Forest (CF) ~\cite{wagerEstimationInferenceHeterogeneous2018a}, a random forest–based approach.

We evaluated AM attacks in this setting using the Twins dataset~\cite{louizosCausalEffectInference2017a}, which studies how birth weight affects infant mortality. The dataset consists of 11{,}400 twin pairs with birth weights under 2 kg and includes 30 covariates. In each pair, the heavier twin is assigned as treated and the lighter as control. The observed mortality rates are 17.69\% (control) and 16.08\% (treatment), yielding a ground truth ATE of $-1.61\%$, indicating slightly higher mortality among lighter twins. As is standard practice~\cite{curthInductiveBiasesHeterogeneous2021,curthReallyDoingGreat2021}, we construct a realistic observation dataset from the ground truth set by including only one of each twin pair, randomly, in the dataset. The dataset is split into training (50\%) and testing (50\%) sets for model fitting and evaluation. We used the python package ``CATENets'' \cite{curthReallyDoingGreat2021} (v0.2.4) and R package ``grf'' (v2.4.0) \cite{atheyGeneralizedRandomForests2019} for the implementation of non-linear ATE estimators.

    In our experiments, the adversary is limited to introducing missing values in covariates (excluding the treatment variable $W$), while the modeler uses either mean or MICE imputation to impute the missing variables. We also tested a doubly robust estimation procedure introduced in \cite{mayerDoublyRobustTreatment2020} introduced specifically for handling missing data in covariates. The adversary's objective is to manipulate the estimated ATE to be 10\%---a drastic shift of approximately 700\%, which falsely suggests that heavier babies have substantially higher mortality.

To implement the BLAMM attack, we used a logistic regression model (conditioning on both $X$ and $W$) as $\hat{\mu}_w(x)$ in the lower-level problem. The upper-level objective minimizes the absolute error between the target ATE and the estimated ATE, i.e., $g(\ttil; \mZ) = |\hat{\tau} - 0.10|$ (See Appendix \ref{app:ate} for its derivation). We introduced missingness in two covariates: gestat (gestational age) and wtgain (weight gain during pregnancy), and parameterized the missingness mechanism using the per-data-point setup. As an additional baseline to the MCAR missingness introduced earlier, we tested an MNAR mechanism that uses the masked variables value through a logistic function to determine the probability its observed  \cite{muzellecMissingDataImputation2020}.

\paragraph{Partial Data Access:} 
In practice, datasets can be aggregated from multiple sources. Therefore, it is of interest to consider settings where an adversary only controls a subset of the full dataset. We model this scenario by assuming that the aggregated dataset contains $N$ rows, of which only $N_0$ rows are provided or influenced by the adversary. We consider four data access regimes where the adversary controls 25\%, 50\%, 75\%, or 100\% of the data (i.e., $N_0/N \in \{0.25, 0.5, 0.75, 1.0\}$). We assume this proportion is known and use it to adapt the BLAMM algorithm accordingly. Under partial data access, the effective missingness mechanism becomes a mixture: rows controlled by the adversary follow the AM-induced missingness pattern, while the remaining rows follow a fully observed (non-adversarial) pattern. This is incorporated into BLAMM as a convex combination of the adversarial and identity mechanisms (see Appendix~\ref{sec:nnparam} for details).

Table~\ref{tab:ate} presents the results across different access regimes. The expected fraction of missing values in the masked covariates (measured over the full dataset) was 7.2\%, 11.1\%, 24.3\%, and 12.5\% for adversary access levels of 100\%, 75\%, 50\%, and 25\%, respectively. Despite reduced access, BLAMM consistently succeeded in misleading various regression-based ATE estimators, producing inflated estimates with the incorrect sign. We found that attacks targeting a mean imputation proxy still inflated the ATE, even when the modeler used more sophisticated methods like MICE imputation or the MIA technique for the CF algorithm (Appendix Table \ref{app:tab:ate_mice}). In contrast, under the baseline MCAR and MNAR attacks, the estimated ATE remained close to the ground truth of -1.61\% (Table \ref{tab:ate}, Appendix Tables \ref{app:tab:mnar},\ref{app:tab:mnar2}). 

These results suggest that AM mechanisms, even when optimized against a simple logistic regression model, can generalize across model classes and remain effective under realistic constraints on adversarial data access.
	
	\section{Relevant Work}
	%Previous data poisoning attacks on linear models require adding a small fraction of adversarially crafted data points to the training set \cite{xiaoFeatureSelectionSecure2015,jagielskiManipulatingMachineLearning2018,wenGreatDispersionComes2021,suvakDesignPoisoningAttacks2022}. Those attacks are not feasible in robust ML settings where the features are cryptographically signed. 
    Bi-level optimization has previously been used to develop insertion-based data poisoning attacks \cite{jagielskiManipulatingMachineLearning2018}, but our formulation differs as it is for the threat model of adversarial missingness, so both the upper and lower objectives are incomparable.
	
	There is little prior work on omission-based attacks. \cite{kdd, koyuncu2024adversarial} are the most relevant, as they develop attacks under the same threat model, but the applicability of their AM attacks is severely limited, as detailed in the introduction. \cite{barashLearnerIndependentTargetedData2020} considers removing complete examples from a dataset, which can be considered an attack on CCA; their approach requires combinatorial optimization, as opposed to our differentiable formulation. \cite{cheng2018non}~shows a semi-random adversary that, by selectively revealing the true values of initially  missing entries, can invalidate the assumptions of non-convex matrix completion and introduce spurious local minimas. 
	
	%to manipulating non-convex matrix completion via revealing existing missing entries \cite{cheng2018non} and introducing back-doors to classification algorithms via removing complete examples \cite{ibm}. 
	
	Another remediation strategy is to jointly model the partially observed variables and the underlying missingness mechanism. Recent work \cite{ipsenNotMIWAEDeepGenerative2021,ma2021identifiable,ghalebikesabiDeepGenerativeMissingness2021} considers learning deep generative models to impute the missing entries. Such approaches make restrictions on the missingness mechanism to ensure identifiability: \cite{ipsenNotMIWAEDeepGenerative2021,ma2021identifiable} assumes the missing value indicators are conditionally independent given the complete observations, and \cite{ghalebikesabiDeepGenerativeMissingness2021} assumes independence of the observed and missing variables given the missingness pattern. While we expect such methods to show robustness to our attack, their assumptions can limit their success. Further, recent surveys of missing data imputation methods suggest that traditional imputation algorithms, including the MICE algorithm, can outperform deep-learning based approaches~\cite{sunDeepLearningConventional2023,wangAreDeepLearning2022a}. Our AM attack showed success against the MICE algorithm's implementation in the popular ``mice'' R-package.

\section{Conclusion}
This work introduces a general and effective framework for adversarial missingness attacks targeting widely used missing data remediation techniques. Our results show that moderate levels of adversarially introduced missingness can suppress feature significance and reverse treatment effect estimates even when the adversary can adversarily modify only a subset of the training data. These findings raise the need to reassess the security implications of current methods for learning with incomplete data.

The core of our approach is a flexible bi-level optimization strategy for constructing AM attacks, applicable to both supervised and unsupervised learning and to a wide range of objective functions. Our framework currently assumes knowledge of the model class and missingness handling method used by the modeler, although there is some evidence that attacks designed for one mitigation mechanism are effective upon others. Extending this framework to settings involving other remediation techniques---including multiple imputation, $k$-nearest neighbor imputation, and generative models---is a promising direction for future work.

%\section{Acknowledgments}

\bibliography{sample}

\appendix
\onecolumn
This supplementary material contains material that could not be included in the main body of the paper due to the space constraints:
\begin{itemize}
\item Figure~\ref{fig:overview} provides a more detailed overview of the process involved in AM attacks.

\item Appendix~\ref{sec:proof} provides a proof of the technical result used in the derivation of the differentiable proxy function for mean imputation.

\item Appendix~\ref{sec:ls} derives a differentiable approximation for regression-based imputation.

\item Appendix~\ref{sec:bileveldetails} provides additional details on how the bi-level optimization problem is solved. The first subsection provides the details of the missingness weighted iterated reweighted least squares solver that is used to minimize the differentiable proxy function for imputation (\eqref{eq:imp_simple}) when the model class is a GLM. The second subsection provides details on the specific functional forms of the adversary's objective (the upper-level objective) for the experiments in Section~\ref{sec:exp}. The third subsection provides a listing of the BLAMM algorithm for solving the bi-level optimization problem, and discusses its computational complexity. The fourth subsection discusses in more detail the two approaches for parameterizing the missingness mechanism, and how the algorithm is adjusted to handle partial access to the training data.

\item Appendix~\ref{app:results} provides additional experimental results. The first subsection provides additional results on the manipulation of p-values of features, provides the details of the experimental setups (neural network training and hyperparameters for the missingness mechanism, and the modeler parameters) for these sets of experiments, and provides results when the data valuation defenses of \cite{just2023lava} and \cite{jia2019efficient} are used against our AM attacks. The second subsection provides additional performance metrics for the GLM experiments. The third subsection provides more results on manipulating ATE under partial data access, and the details of the experimental setups for these sets of experiments.
\end{itemize}

\begin{figure*}[t!]
	\centering
	\includegraphics[width=1\linewidth]{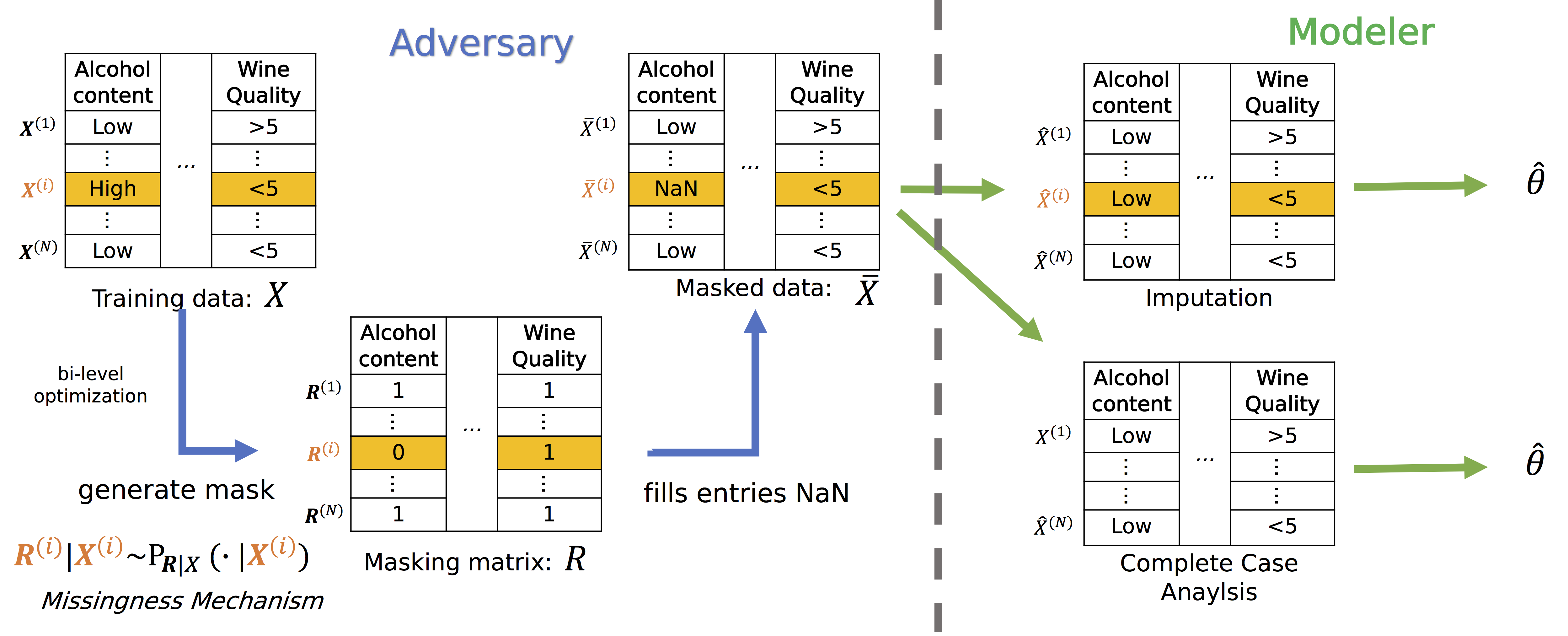}
	\caption{In AM attacks, the adversary uses the samples in the training set to design a missingness mechanism $\p_{\Rb|\Z}$, given knowledge of the modeler's technique for mitigating missing data. The adversary then samples a masking matrix from the missingness mechanism, replaces the indicated entries with $\text{NaN}$s, and conveys this poisoned dataset to the modeler. The modeler applies CCA or imputation to the partially observed dataset, then proceeds to learn the model.}
	\label{fig:overview}
\end{figure*}

\section{Proofs}\label{sec:proof} 
\begin{proposition}\label{prop:mean}
	The expected value of a variable $\Z_j$ conditioned on $\Rb_j=1$ can be expressed in terms of the \mm $\P_{\Rb\mid \Z}$ as follows: \[\E[\Z_j \mid \Rb_j = 1]=\frac{1}{\P_{\Rb_j}(1)}\E[\Z_j \sum_{\forall \rb, \rb_j=1} \P_{\Rb\mid \Z}(\rb \mid \Z)]\]
\end{proposition}
\begin{proof}
	\begin{equation}\label{eq:impute_proof}
		\begin{split}
			\E[\Z_j \mid \Rb_j = 1]&=\sum_{\z_j} \P_{\Z_j\mid \Rb_j}(\z_j\mid 1)\z_j=\frac{1}{\P_{\Rb_j}(1)}\sum_{\z_j} \P_{\Z_j,\Rb_j}(\z_j,1)\z_j\\
			\intertext{The distribution of $\Z_j$ and $\Rb_j$ is equals to the marginalized joint distribution over all variables}
			&=\frac{1}{\P_{\Rb_j}(1)}\sum_{\z}\sum_{\forall \rb, \rb_j=1} \P_{\Z, \Rb}(\z,\rb)\z_j\\
			\intertext{Using the chain rule it becomes}
			&=\frac{1}{\P_{\Rb_j}(1)}\sum_{\z}\sum_{\forall \rb, \rb_j=1} \P_{\Z}(\z) \P_{\Rb\mid \Z}(\rb \mid \z) \z_j\\
			&=\frac{1}{\P_{\Rb_j}(1)}\sum_{\z}\P_{\Z}(\z) \z_j [\sum_{\forall \rb, \rb_j=1} \P_{\Rb\mid \Z}(\rb \mid \z)]\\						&=\frac{1}{\P_{\Rb_j}(1)}\E[\Z_j \sum_{\forall \rb, \rb_j=1} \P_{\Rb\mid \Z}(\rb \mid \Z)]\\
		\end{split}
	\end{equation}
\end{proof}

\section{Differentiable Proxy for Regression-based Imputation}\label{sec:ls}
Alternatively the modeler may use the dependence between the missing and the observed entries to impute the missing entries. One straightforward approach known as conditional mean imputation \cite{littleStatisticalAnalysisMissing2002} uses a linear model to regress the missing variables upon the observed variables. Conditional mean estimation and its  Bayesian formulations are also used within more sophisticated imputation methods such as the popular Multiple Imputation by Chained Equations (MICE) algorithm~\cite{buurenMiceMultivariateImputation2011a}.

We assume that the adversary uses a missingness mechanism that restricts the missingness to the features in $\mathcal{M}\subseteq\{1,\dots,d\}$, and let $\overline{\mathcal{M}}$ denote the features that are always observed\footnote{If multiple such patterns of missingness ($\mathcal{M}, \overline{\mathcal{M}}$) occur with nonzero probability under the AM mechanism, the steps described now should be repeated for each pattern. }. The modeler expresses missing observations of the $j$th variable ($j\in\mathcal{M}$) as linear combinations of the completely observed variables $\overline{\mathcal{M}}$. 

To estimate the linear coefficients for a fixed pattern of missingness using least squares, the modeler uses the subset of rows where the $j$th column is observed, namely ${\mathcal{S}_j=\{i:\rb_j^{(i)}=1\}}$. The linear predictor of the $j$th variable is found by solving 
\begin{equation}\label{eq:ls}
	\argmin_{\bb} \frac{1}{|\mathcal{S}_j|} \sum_{i\in \mathcal{S}_j}  (\z_j^{(i)}-\inner{\z_{\overline{\mathcal{M}}}^{(i)}}{\bb})^2.
\end{equation}
%Notice that this least squares objective can be viewed as a CCA objective (given in \eqref{eq:modeler_cca_emp}) applied to the subset of columns $\overline{M}\cup\{j\}$. 
The linear coefficients resulting from a single realization of the dataset are not differentiable with respect to the missingness mechanism. Therefore, to capture the dependence on the missingness mechanism, we use an approximate set of coefficients $\hat{\bb}(\phi)^{j}\in\real^{|\overline{\mathcal{M}}|}$: \begin{equation}\label{eq:tilde_xh_lr}
	\ztphi_j=
	\begin{cases}
		\z_j^{(i)}, &\text{if } \rb_j=1\\
		\inner{\z_{\overline{\mathcal{M}}}^{(i)}}{\hat{\bb}(\phi)^{j}},         &\text{otherwise}
	\end{cases}
\end{equation}
where the notation $\hat{\bb}(\phi)^{j}$ is used to highlight the fact that the regression coefficients are a function of the missingness mechanism. These coefficients come from an empirical approximation of the asymptotic form of~(\eqref{eq:ls}). This approximation corresponds to solving a weighted least squares problem over the dataset, where the weight of the $i$th sample is the probability of observing the $j$th variable in the $i$th sample. We next proceed to the derivation of $\hat{\bb}(\phi)^{j}$ and its dependence on $\vphi$. 

In regression-based imputation, the modeler uses the always observed columns $\overline{\mathcal{M}}$ to regress the $j$th column. To fit the regression model, only the rows where the $j$th column is complete can be used, ${\mathcal{S}_j=\{i:\rb_j^{(i)}=1\}}$. The corresponding score function is the squared error $J(\z;\theta)=(\z_j-\inner{\z_{\overline{\mathcal{M}}}}{\bb})^2$, so asymptotically the objective used to learn the regression weights has the form  
%\P_{\Rb\mid \Z}(\rb \mid \z^{(i)},\phi)
\begin{equation}\label{eq:ls2}
	\frac{1}{|\mathcal{S}_j|}\sum_{i\in \mathcal{S}_j} J(\z^{(i)};\theta)\overset{P}{\rightarrow}\E[J(\Z;\vtheta)\mid \Rb_j=1]
\end{equation}
The expectation is conditioned on the $j$th variable being observed. Following the same steps we used while deriving the CCA differentiable objective (\eqref{eq:cca_weight}), we obtain the the following differentiable approximation to the least squares imputation:
\begin{equation}
	\hat{\bb}(\vphi)^{j}=\argmin_{\bb} \frac{1}{N}\sum\limits_{i=1}^{N}\P_{\Rb_j\mid \Z}(1\mid \z^{(i)};\vphi) J(\z^{(i)};\vtheta) 
\end{equation}
It is a weighted least squares problem with the weight of the $i$th sample given by $\P_{\Rb_j\mid \Z}(1\mid \z^{(i)};\vphi) $.
%	\propto \E[\P_{\Rb\mid \Z}(\mathbf{1}\mid \Z;\vphi) J(\Z;\vtheta)]

We denote the corresponding diagonal weighting matrix as $\mW(\vphi)\in\real^{N \times N}$ with ${\mW_{i,j}(\vphi)=\P_{\Rb_j\mid \Z}(1\mid \z^{(i)};\vphi)}$ if $i=j$ and $\mW_{i,j}(\vphi)=0$ otherwise. The optimal regression coefficients are 
\[\hat{\bb}(\vphi)^{j}=(\mA^T\mW(\vphi) \mA)^{-1}\mA^T\mW(\vphi) \mathbf{b},\] where $\mA=\mZ_{:,\overline{\mathcal{M}}}$  denotes the completely observed columns of the dataset and $\mathbf{b}=\mZ_{:,j}$.

Using a sample from (\eqref{eq:tilde_xh_lr} ) in the proxy objective $\tilde{f}(\vtheta, \vphi;\mZ)$ from (\eqref{eq:imp_simple}) gives a proxy objective for modelers using conditional mean imputation that is differentiabile with respect to $\vphi$.

\section{Solving the Bi-level Problem}
\label{sec:bileveldetails}

In this section, we present additional details regarding Section \ref{sec:solution}.

\subsection{Solving the lower level optimization problem for GLMs}\label{sec:gradient}

Assume the modeler is interested in fitting a generalized linear model (GLM) to the conditional distribution of a response random variable $Y$, given the vector of features $\X$. Following Almudevar~\cite{almudevar2021theory}, GLMs in canonical form are parameterized using a linear predictor term $\boldsymbol{\eta} = \inp{\vtheta}{\x}$ and have probability density (or mass) function
\begin{equation}\label{eq:glm}
	\begin{split}
		p(y\mid\x; \vtheta)
		&= h(y, \sigma) \exp\left(\frac{\boldsymbol{\eta} y - A(\eta)}{\sigma}\right)
	\end{split}.
\end{equation}

Here, $\vtheta$ denotes the regression coefficients, $\sigma$ is a dispersion parameter, $A$ is the partition function and $h$ is a base measure~\cite{almudevar2021theory}. In the following, the dispersion parameter is constant unless otherwise stated, and is therefore omitted from the notation. 

As we will show the gradient of the modeler's proxy objective under imputation  (\eqref{eq:imp_simple}) can be written as a weighted sum of the gradient of the GLM log-likelihood. Following \cite{murphy}, the gradient of the log-likelihood is given as:

\begin{equation}\label{eq:glm2}
	\begin{split}
		\nabla_{\vtheta}\log \P(y\mid \x;\vtheta)
		&= \frac{1}{\sigma^2}\x (y-A'(\inp{\vtheta}{\x}))
	\end{split}
\end{equation}

where $A'(\inp{\vtheta}{\X})$ the first derivative of the partition function. The gradient of the modeler's objective under imputation, \eqref{eq:imp_simple}, has an exact summation over $\rb$ and to simplfy the notation, ${\omega_{i,\rb}(\vphi)\defeq P_{\Rb\mid \Z,Y}(\rb \mid \z^{(i)},y^{(i)},\vphi)}$. Taking the gradient and using \eqref{eq:glm2} leads to

\begin{equation}
	\begin{split}
		\nabla_{\vtheta} \ft(\vtheta,\vphi)&=\sum_{i=1}^{N}\sum_{\rb\neq \mathbf{0}} \omega_{i,\rb}(\phi)\nabla_{\vtheta}\log \P(y^{(i)}\mid \hat{\x}^{(i,\rb)}; \vtheta)\\
		&=\frac{1}{\sigma^2}\sum_{\rb\neq \mathbf{0}}\sum_{i=1}^{N}\omega_{i,\rb}(\phi) \hat{\x}^{(i,\rb)} (y^{(i)}-A'(\inp{\vtheta}{\hat{\x}^{(i,\rb)}}))
	\end{split}
\end{equation}

Similarly, we derive the Hessian of the \eqref{eq:imp_simple}:

\begin{equation}
	\begin{split}
		\nabla_{\vtheta}^2 \ft(\vtheta,\vphi) &= \nabla_{\vtheta}[\sum_{i=1}^{N}\sum_{\rb\neq \mathbf{0}} \omega_{i,\rb}(\phi)\nabla_{\vtheta}\log \P(y^{(i)}\mid \hat{\x}^{(i,\rb)}; \vtheta)]\\
		&=\frac{1}{\sigma^2}\sum_{\rb\neq \mathbf{0}}\sum_{i=1}^{N}\omega_{i,\rb}(\phi) \nabla_{\vtheta} [\hat{\x}^{(i,\rb)} (y^{(i)}-A'(\inp{\vtheta}{\hat{\x}^{(i,\rb)}}))]\\
		&=\frac{-1}{\sigma^2}\sum_{\rb\neq \mathbf{0}}\sum_{i=1}^{N}\omega_{i,\rb}(\phi)A''(\inp{\vtheta}{\hat{\x}^{(i,\rb)}})\hat{\x}^{(i,\rb)}\hat{\x}^{{(i,\rb)}^T}
	\end{split}
\end{equation}

Using these equations we can utilize iterated reweighted least squares to minimize $\ft$.
\begin{algorithm}[h!]
	\caption{Missingness weighted IRLS Algorithm. It is the modified version of Algorithm 8.1 given in \cite{murphy}.}\label{alg:2}
	\begin{algorithmic}
		\STATE {\bfseries Input: $\epsilon,\phi, \{\x^{(i)},y^{(i)}\}_{i=1}^N,$}
		%		\STATE $\tat{0}\leftarrow\tr$
		\STATE $\thetat{0} \leftarrow \text{Initialize}$
		\STATE $t \leftarrow 0$
		\WHILE{$|\ft(\thetat{t},\vphi)-\ft(\thetat{t-1},\vphi)|\geq \epsilon |\ft(\thetat{t-1},\vphi)|$} 
		\STATE $\mathbf{g}^{(t)}\leftarrow \nabla_{\vtheta} \ft(\thetat{t},\vphi)$
		\STATE $\mathbf{H}^{(t)}\leftarrow \nabla_{\vtheta}^2 \ft(\thetat{t},\vphi)$
		\STATE $\mathbf{d}^{(t)}\leftarrow -\mathbf{H}^{{(t)}^{-1}}\mathbf{g}^{(t)}$
		\STATE $\thetat{t+1} \leftarrow \thetat{t}+\mathbf{d}^{(t)}$
		\STATE $t\leftarrow t+1$
		\ENDWHILE
		\STATE {\bfseries Output: $\thetat{t}$} 
	\end{algorithmic}
\end{algorithm}

Note that when the GLM distribution is Gaussian one step is sufficient for convergence. Notice, the gradient of the proxy objective for CCA can be seen as a special case of the above equations where only the fully observed masks enter into the summation. That is because the additional normalization constant $\P_{\Rb;\vphi}(\mathbf{1})$ term in the $\eqref{eq:cca_weight}$ acts as a constant for the lower level problem and can be omitted from the expression for $\tilde{f}$. Therefore, the weighted IRLS algorithm can be used to solve the lower level problem for attacking CCA.

\subsubsection{Computational Complexity} \label{sec:cc1}
In Algorithm \ref{alg:2}, each step requires updating the imputation parameters, computing the gradient and the Hessian of the missingness weighed objective function. Each of these steps requires enumerating over all possible combinations of missingness patterns of the masked features\footnote{assuming at least one feature is observed}, i.e. $\mask$. The  computational complexity increases exponentially with the size of this set. Hessian computation requires for each missingness pattern an outer product and therefore dominates the complexity and results in an $O(2^{|\mask|}Nd^2)$ complexity for each iteration of Algorithm~ \ref{alg:2}.

\subsection{Specifying the Upper Level Objective}\label{sec:klmin}

In this section, we give details on how the upper level (See Section \ref{sec:solution}) is specified in our experiments.

\subsubsection{Manipulating p-values of features}
To select the adversarial model while setting the target parameter to zero, we solved the following problem:
\begin{align}
	\ta \in \argmin_{\vtheta \ :\  \theta_{t}=0}  \sum_{i=1}^N -\log \P(y^{(i)}\mid \x^{(i)};\vtheta) \label{eq:KLmin}.
\end{align}
This optimization problem fits a GLM that maximizes the likehood subject to the constraint of not using the $j$th feature.

For the distance measure in BLAMM, we used the KL-divergence between the adversarial GLM and the simulated model $\ttil$,
\begin{equation*}
	g(\ttil;\mZ)=\frac{1}{N}\sum_{i=1}^N{\text{D}_{\textrm{KL}}(\P_{Y\mid \X^{(i)};\ta}\,\|\, \P_{Y\mid \X^{(i)};\ttil})}
\end{equation*}

\subsubsection{Manipulating ATE under Partial data access}\label{app:ate}

Given the lower-level problem solution $\ttil$ that parametrizes the conditional distribution of the outcome given the covariates and the treatment variable, we estimate the ATE for a binary outcome using a regression estimator as follows:
\[\hat{\tau}(\ttil;\mZ)=\frac{1}{N}\sum_{i=1}^N \P(y=1\mid \x^{(i)},W=1;\ttil)-\P(y=1\mid \x^{(i)},W=0;\ttil)\]

Next, we use the estimated ATE to specify the upper-level problem objective function by:
\[g(\ttil;\mZ)=|\hat{\tau}(\ttil;\mZ)-\tau_{\alpha}|\] where $\tau_{\alpha}=0.1$.

\subsection{The BLAMM Algorithm}\label{sec:blamm}

The gradient of the upper level problem in~(\eqref{eq:lamm}) takes the following form\footnote{$\nabla_i\ell(.,.)$ denotes the gradient with respect to the $i$th argument.}, by the chain rule~\cite{gouldDifferentiatingParameterizedArgmin2016}:
\[\nabla\ell(\vphi,\g(\vphi))=\gx \ell(\vphi,\g(\vphi)) + \mJ_{\ttil}(\vphi)^T 
\gy \ell(\vphi,\g(\vphi)).
\]
To calculate the Jacobian of the solution to the inner problem with respect to $\vphi$, we note that $\ttil(\vphi)$ is a zero of the gradient of $\ft$ with respect to $\vtheta$, i.e. $\nabla_1\ft$. We thus employ the Implicit Function Theorem, as is common in bi-level optimization. To do so, we partition the Jacobian of $\nabla_1\ft$ into two blocks:
$\mJ_{\nabla_1\ft}(\vtheta,\vphi)=[ \mA(\vtheta,\vphi)\mid \mB(\vtheta,\vphi)],$ where $\mA(\vtheta,\vphi)$ is the Hessian of the objective $\ft$ with respect to $\vtheta$ and $\mB(\vtheta,\vphi)$ contains the remaining entries. Following  Lemma~3.2 and Equation~14 of~\cite{gouldDifferentiatingParameterizedArgmin2016} we find that, under under certain regularity conditions on $\ft$,
\begin{equation}\label{eq:cca_grad}
	\mJ_{\ttil}(\vphi)=-\mA(\ttil,\vphi)^{-1}\mB(\ttil,\vphi).
\end{equation}

These regularity conditions are listed in \cite{lorraineOptimizingMillionsHyperparameters2019} as: differentiability of the upper-level problem $\ell$, twice differentiability of the lower-level problem $\ft$, invertibility of $\ell$'s Hessian at $\g(\vphi)$, and differentiability of $\g(\vphi)$. Although, using this formulation does not require the lower-level problem to have a unique solution, multiple global optima can pose practical challenges \cite{gouldDifferentiatingParameterizedArgmin2016}. Please see the survey paper \cite{zhangIntroductionBilevelOptimization2023} for alternative bi-level optimization methods that can handle lower-level problems with multiple optimal solutions.

We refer to the resulting algorithm that learns the AM mechanism by using gradient descent on~\eqref{eq:lamm} as the Bi-level Formulation for Learning AM Mechanisms (BLAMM). A listing is provided in Algorithm~\ref{alg:blamm}. 

\begin{algorithm}[t]
	\caption{BLAMM Algorithm for gradient descent based training of a missingness mechanism parameterized by $\vphi$.}
	\label{alg:blamm}
	
	\begin{algorithmic}
		\STATE {\bfseries Input:}: $\mZ, \lambda,\ta,\eta$, max\_steps
		\FOR{t=1,2, \dots, max\_steps} 
		\STATE $\ttil \leftarrow \argmin_{\vtheta} \ft(\vtheta,\vphi)$\qquad \COMMENT{Learn a model given the AM mechanism and modeler objective.}
		\STATE $\ell(\vphi,\vtheta)\leftarrow g(\ttil;\mZ)+ \lambda\times  \Omega(\vphi;\mZ)$ \COMMENT{Evaulate the adversary's goals.}
		\STATE $\mJ_{\ttil}(\vphi)\leftarrow -\mA(\ttil,\vphi)^{-1}\mB(\ttil,\vphi)$ \COMMENT{Compute the Jacobian of the simulated parameter.}
		\STATE $\vphi \leftarrow \vphi-\eta [\gx \ell(\vphi,\g(\vphi)) + \mJ_{\ttil}(\vphi)^T 
		\gy \ell(\vphi,\g(\vphi))]$ \COMMENT{ Update the AM mechanism.}
		\ENDFOR	
		\STATE {\bfseries Output: $\vphi$} 
	\end{algorithmic}
\end{algorithm}
\subsubsection{Computational Complexity}\label{sec:cc2}
Algorithm \ref{alg:blamm} iteratively updates the missingness mechanism using gradient descent. Because the loss function involves the best response from the lower-level problem, each step requires solving the lower level problem and computing the Jacobian of the simulated parameters. Excluding the cost of computing the loss function and solving the lower level problems (as these costs are problem dependent),
the computational complexity of each iteration of Algorithm~\ref{alg:blamm} is $O(d^3+d^2|\vphi|)$ where $|\vphi|$ denotes the number of parameters in the missingness mechanism. 

\subsection{Parameterizing the Missingness Mechanism}\label{sec:nnparam}

The missingness mechanism allows missingness only in the the masked variables $\mask$, i.e. all missingness patterns with non-zero probability has $\rb_{-\mask}=1$. To determine the probability that $\rb_\mask$ takes on a particular binary mask we considered two options:
 
\paragraph{Neural Networks}
Following \cite{koyuncu2024adversarial}, we parameterize the missingness distribution by using a neural network to assign a probability to observing the masked variables $\mask$ in a given instance $\z$. Let $\vz(\z,\vphi)\in \real^{2^{|\mask|}}$ denote the vector of logits from a neural network that takes $\z$ and is parameterized by $\vphi$. Denote the mapping of the binary mask vector, $\rb_{\mask}$, to the output units with a binary-to-decimal convertor $\gamma(\rb_{\mask})$\footnote{As an example, imagine $\gamma((1,0,1))=5$}. 
\begin{equation*}
	\p_{\Rb|\Z}(\rb|\z;\vphi)= 
	\begin{cases}
		\text{softmax}(\vz(\z,\vphi))_{\gamma(\rb_{\mask})},&\text{if } \rb_{-\mask}=1\\
		0,              &\text{otherwise}
	\end{cases}.
\end{equation*}
\paragraph{Per-data-point Setup}
For each instance $\z^{(i)}$ a separate vector $\vphi^{(i)}\in \real^{2^{|\mask|}}$ is learned. The missingness mechanism corresponding to the i'th data point is given by:
\begin{equation*}
	\p_{\Rb|\Z}(\rb|\z^{(i)};\vphi)= 
	\begin{cases}
		\text{softmax}(\vphi^{(i)})_{\gamma(\rb_{\mask})},&\text{if } \rb_{-\mask}=1\\
		0,              &\text{otherwise}
	\end{cases}.
\end{equation*}

%\paragraph{title}

%This missingness mechanism assigns nonzero probability only to masks $\rb$ in which the non-masked variables are all observed, and the probability that $\rb_\mask$ takes on a particular binary mask is determined by the neural network.
%\ref{xx}
\paragraph{Adjusting for Partial Access}
In our experiments (Section \ref{sec:ate}), we considered the setting where the adversary cannot access all $N$ rows and can only introduce missingness to a subset of $N_0$ of the rows. This results in a mixture model, where if a row falls into the portion adversary cannot access, all entries of the row would be observed \footnote{Assuming original data does not contain missingness} and if it falls into the portion adversary can access, it would be masked according to the AM mechanism. If we assume the accessible rows are selected independent from the data, the resulting missingness mechanism of the dataset becomes a convex combination of the identity missingness mechanism and the AM mechanism i.e.,

\begin{equation}\label{eq:mix}
	\p_{\Rb|\Z}(\rb|\z;\vphi)= (1-\frac{N_0}{N}) 
	I(\rb = \mathbf{1}) + \frac{N_0}{N}\p'_{\Rb|\Z}(\rb|\z;\vphi)
\end{equation}
where $\p'_{\Rb|\Z}(\rb|\z;\vphi)$ is the learned AM mechanism and $I(\rb = \mathbf{1})$ is the indicator function which is equals to 1 if all variables are observed, i.e., $\rb = \mathbf{1}$ and zero otherwise. While using the BLAMM algorithm with the
the accessible portion of the dataset, the mixture form of the missingness mechanism, $\p_{\Rb|\Z}$, is used. After training is completed, to mask the corresponding entries of the accessible dataset the AM mechanism $\p'_{\Rb|\Z}$ is directly used.

%Assuming these subset of rows were randomly selected, the effective missingness mechanism is a mixture of the adversarial missingness mechanism and 

%the probability of observing a particular missingness pattern becomes a mixture of two missingness mechanisms:

%\ref{yy}
\section{Expanded Results and Discussion}\label{app:results}

In Appendix \ref{app:results}, we present additional experimental results.
\subsection{Manipulating p-values of features}
In this section, we present additional results from Section \ref{sec:attack1}.
\subsubsection{Dataset Description}\label{app:data}

We used two UCI datasets: $wine-quality$, $german-credit$, and two datasets provided in the scikit-learn package: $ca-housing$, $diabetes$. The first UCI dataset $wine-quality$ is introduced in \cite{Cortez2009ModelingWP} and has 6497 samples with 11 continuous features. The original response, wine quality score, is a categorical variable. To convert the problem into a binary classification problem (low, and high) we used wine quality score $>5$ as the positive class and $\leq 5$ as the negative class. Second UCI dataset $german-credit$ \cite{german_credit} has 1000 observations with 20 mixed-type features. The dataset is a binary classification problem and goal is to predict credit risk (good or bad). We have used one-hot-encoding for the discrete features with the $drop="first"$ option for numerical stability \footnote{See OneHotEncoder in scikit-learn \cite{scikit-learn}}. That resulted in 48 features in total. The first of the two scikit-learn datasets: California housing ($ca-housing$) is introduced in \cite{pace1997sparse} and contains 20,640 samples. It has 8 continuous features and the goal is to regress the median house value. The second dataset we accessed from scikit-learn, $diabetes$  comprises 442 samples and is also a regression problem \cite{lars}. It contains 9 continuous features and a single binary feature. The continuous response variable measures the progression of the disease \cite{lars}.

We used a random subset of each dataset (90\% for $\text{german-credit}$ and 80\% for the others) as the training set that the \noun can access, and kept the remaining fully observed portion for use as an auditing set. The $\text{german-credit}$ dataset contains categorical features, with some categories observed in only a few data points. In earlier experiments, we observed one-hot-encoding the categories created a numerical instability in the BLAMM algorithm and estimation of the GLM's coefficients. Therefore, in the $\text{german-credit}$ we preferred a slightly larger training set proportion to allow BLAMM the opportunity to stabilize on these categorical features.

As adversarial targets, we selected \verb|medInc| for ca-housing, \verb|alcohol| for wine-quality, \verb|sex| for diabetes, and \verb|installment rate| in the german-credit dataset.

\begin{table}[h]
	\centering
	\caption{The four data sets used and the selected target variables. The (type) indicates the type of the response variable in the first column and the \adj target variable in the second as c:continous, b:binary. The third column indicates the value of the target coefficient and its p-value. The last two columns correspond to the auditing set score of the complete data estimate $\tr$ and the adversarial parameter $\ta$.}
	% Table generated by Excel2LaTeX from sheet 'datasets'
\begin{tabular}{|l|p{5.5em}|r|r|r|}
\hline
\multicolumn{1}{|c|}{Dataset (type)} & \multicolumn{1}{c|}{Target (type)} & \multicolumn{1}{p{5.165em}|}{Coef. $ \vtheta_{p,t}$ \newline{}(p-value)} & \multicolumn{1}{p{5.5em}|}{Test \newline{}Score ($\tr$)} & \multicolumn{1}{p{3.835em}|}{Test Score($\ta$)} \bigstrut\\
\hline
\hline
ca-housing (c) & \multicolumn{1}{l|}{medInc (c)} & 0.45 (0) & NMSE: 0.424 & 0.594 \bigstrut\\
\hline
wine-quality (b) & \multicolumn{1}{l|}{alcohol (c)} & 0.9 (4E-43) & ACC: 0.724 & 0.722 \bigstrut\\
\hline
diabetes (c) & \multicolumn{1}{l|}{sex (b)} & -23.06 (5E-4) & NMSE: 0.541 & 0.557 \bigstrut\\
\hline
german-credit (b) & installment \newline{}rate (c)  & 0.32 (4E-4) & ACC: 0.8 & 0.78 \bigstrut\\
\hline
\end{tabular}%

	\label{tab:dataset}%
\end{table}%

\subsubsection{Setup: Neural Network Training and Hyper-parameters}\label{app:results_nn}

We implemented BLAMM using tensorflow by defining a custom keras model. For the input to the NN (the conditioned variables of the missingness mechanism), we used all available variables, $\z=(\x,y)$. While training the NN, the inputs (generally except the response) are scaled using a ``standard scaler''.  While fitting the CCA attack on the wine-quality dataset we have observed numerical instability (See Figure \ref{fig:wine}), so we added a small amount of $\ell-2$ regularization (1e-7) to the lower-level problem. For the regression-based imputation attack, we trained a CCA attack during the first 60\% of the total epochs and then switched to fine tuning for the regression-based imputation attack. After the switch, the learning rate is reduced by 100 times.

We used a Intel(R) Core(TM) i7-6700K CPU @ 4.00GHz with 31Gi of RAM to run our experiments. Operating system is Ubuntu 20. The longest training took less than 30 minutes on the CPU.
% Table generated by Excel2LaTeX from sheet 'edit'
%\usepackage{graphicx}
%\usepackage{multirow}
%\usepackage{booktabs}
%\usepackage{xcolor}
%\usepackage{bigstrut}

\begin{table}[htbp]
	\centering
	\caption{The hyper-parameters used for training the neural networks used for the different datasets and attacks. The column $\lambda$ (upper) refers to the regularization controlling the missing data rate in the upper level problem, while $\lambda$ (lower) refers to the regularization added to the lower level problem GLM problem. `lr' indicates the learning rate.}
	% Table generated by Excel2LaTeX from sheet 'param'
\begin{tabular}{|l|c|l|l|l|l|}
\hline
Data  & \multicolumn{1}{l|}{Attack} & lr    & epochs & $\lambda$ (upper) & $\lambda$ (lower) \bigstrut\\
\hline
ca-housing & \multirow{3}[6]{*}{mean} & 0.01  & 200   & 0.01  & 0 \bigstrut\\
\cline{1-1}wine-quality &       & 0.01  & 200   & 0.01  & 0 \bigstrut\\
\cline{1-1}german-credit &       & 0.01  & 200   & 0.01  & 0 \bigstrut\\
\hline
ca-housing & \multirow{4}[8]{*}{cca} & 0.01  & 600   & 0.05  & 0 \bigstrut\\
\cline{1-1}wine-quality &       & 0.01  & 300   & 0.01  & 1e-07 \bigstrut\\
\cline{1-1}german-credit &       & 0.01  & 200   & 0.01  & 0 \bigstrut\\
\cline{1-1}\multicolumn{1}{|l|}{diabetes} &       & 0.01  & 200   & 0.01  & 0 \bigstrut\\
\hline
ca-housing & \multirow{3}[6]{*}{linear} & 0.01  & 1000  & 0.05  & 0 \bigstrut\\
\cline{1-1}wine-quality &       & 0.01  & 500   & 0.01  & 1e-07 \bigstrut\\
\cline{1-1}german-credit &       & 0.01  & 333   & 0.01  & 0 \bigstrut\\
\hline
\end{tabular}%

	\label{tab:hyper}%
\end{table}%

\subsubsection{Setup: Modeler Parameters} \label{sec:modeler}

We used the statsmodels package \cite{seabold2010statsmodels} to fit the GLMs with the  default parameters and initialized IRLS with the zero vector \cite{seabold2010statsmodels}. We used scikit-learn \cite{scikit-learn} for mean imputation and to use the MICE package from R~\cite{buurenMiceMultivariateImputation2011a}, we utilized the rpy2 python interface. We used the $mice$ function from the \mice package with the default parameters and 5 maximum iterations. We used a single imputation of the MICE algorithm but varied the randomness seed across the different masks. We accessed the imputed data using $\$imp$ and fed it into a GLM solver. Pooling the results instead of using a single imputation can make the attacks less effective.

\subsubsection{Results: Missingness Rate}

We report the missingness rate BLAMM converged to in different attack settings in Table \ref{tab:rate}. The missingness rate is implicitly controlled by the regularization parameter $\lambda$. BLAMM learns a missingness mechanism with the two desiderata of being adversarially successful and incurring a low missingness rate. Hence, the observed missingness rate depends on the distribution of the dataset, the target coefficient, and the target missing data remediation method. We have used $\lambda=0.01$ in all experiments except those with the ca-housing dataset. For that dataset, we observed that $\lambda=0.01$ resulted in a high amount of missing data, so we used a higher value of $\lambda=0.05$ for the complete case analysis and linear regression attacks (see Table \ref{tab:hyper} column " (upper)"). 

 The mean imputation attack converged to a lower missingness rate compared to the CCA and regression-based imputation attacks, suggesting that mean imputation may present an easier target for manipulation.

\begin{table}[h]
	\centering
	\caption{The missingness rate of the target variable in the learned \mm in different attack types and datasets.}
	% Table generated by Excel2LaTeX from sheet 'rate'
\begin{tabular}{|l|l|l|l|l|}
\hline
\multicolumn{1}{|p{3.085em}|}{Attack \newline{}Type} & \multicolumn{1}{c|}{ca-housing} & \multicolumn{1}{c|}{wine-quality} & \multicolumn{1}{c|}{german-credit} & \multicolumn{1}{c|}{diabetes} \bigstrut\\
\hline
mean  & 18.1  & 5.5   & 4.5   & - \bigstrut\\
\hline
cca   & 40.2  & 14.9  & 9.2   & 12.5 \bigstrut\\
\hline
linear & 41.1  & 15.3  & 9.0   & - \bigstrut\\
\hline
\end{tabular}%
		
	\label{tab:rate}%
\end{table}%

%\subsubsection{Attack Setup and Additional Commentary}
%
%
%	In the ca-housing dataset, we observed that the target coefficient was statistically significant (average $p<$0.001) when using the CCA attack, but statistically insignificant when the attacker used the \mm designed for linear imputation (average $p$=0.559). 

\subsubsection{Results: Regression-based Imputation}\label{app:results_lr}

We present the expanded experiments in Section \ref{sec:attack1}. We first, tested the attack on regression-based imputation on the mean imputation, CCA, and MICE imputation algorithms. Next, we tested all three attack types on regression-based imputation (See Table \ref{tab:linear}).

We observed specialization can improve the attack performance. In the ca-housing dataset, when the modeler is using linear regression imputation, the target coefficient was statistically significant (average p <0.001) under the CCA attack, but statistically insignificant when the attacker used the missingness mechanism designed for linear imputation (average p=0.559) (See, Table~\ref{tab:pvalue}).

\begin{table}[h!]
	\centering
	\caption{The average (over 20 trials) normalized $\ell_1$ norm of the difference between the modeler-estimated coefficients and the adversarial coefficients, i.e.
		$||\hat{\bm{\theta}}-\bm{\theta}_\alpha||_1/||\bm{\theta}_\alpha||_1$
		%			 $\tfrac{||\hat{\bm{\theta}}-\bm{\theta}_\alpha||_1}{||\bm{\theta}_\alpha||_1}$
		and its standard deviation (denoted as $\pm$). If the target coefficient resulting from the attack is insignificant on average (i.e. average $p$-value $>$ 0.05, Table~\ref{tab:pvalue} in the Appendix \ref{app:results}), this is indicated using a $\checkmark$. \newline}
		\resizebox{1\columnwidth}{!}{%
 % Table generated by Excel2LaTeX from sheet 'dist_a_short'
\begin{tabular}{l|l|l|l|l|l|l|l}
\multicolumn{2}{c|}{Modeler/Attacker} & \multicolumn{1}{l}{mean/linear} & \multicolumn{1}{l}{cca/linear} & \multicolumn{1}{l}{mice/linear} & \multicolumn{1}{l}{linear/linear} & \multicolumn{1}{l}{linear/mean} & linear/cca \bigstrut[b]\\
\hline
\hline
ca-   & BLAMM  & \textbf{0.02$\pm$0.0 (\checkmark)} & \textbf{0.04$\pm$0.0 (\checkmark)} & \textbf{0.21$\pm$0.0} & \textbf{0.03$\pm$0.0 (\checkmark)} & 1.34$\pm$0.0 & \textbf{0.16$\pm$0.0} \bigstrut[t]\\
housing & MCAR  & 0.45$\pm$0.0 & 1.17$\pm$0.0 & 1.16$\pm$0.0 & 1.58$\pm$0.0 & \textbf{1.32$\pm$0.0} & 1.57$\pm$0.0 \bigstrut[b]\\
\hline
wine- & BLAMM  & \textbf{0.02$\pm$0.0 (\checkmark)} & \textbf{0.05$\pm$0.0 (\checkmark)} & \textbf{0.18$\pm$0.0} & \textbf{0.05$\pm$0.0 (\checkmark)} & \textbf{0.74$\pm$0.0} & \textbf{0.05$\pm$0.0 (\checkmark)} \bigstrut[t]\\
quality & MCAR  & 0.47$\pm$0.0 & 0.86$\pm$0.0 & 0.68$\pm$0.0 & 1.03$\pm$0.0 & 0.92$\pm$0.0 & 1.02$\pm$0.0 \bigstrut[b]\\
\hline
german- & BLAMM  & \textbf{0.01$\pm$0.0 (\checkmark)} & \textbf{0.09$\pm$0.0 (\checkmark)} & \textbf{0.01$\pm$0.0 (\checkmark)} & \textbf{0.01$\pm$0.0 (\checkmark)} & \textbf{0.02$\pm$0.0 (\checkmark)} & \textbf{0.01$\pm$0.0 (\checkmark)} \bigstrut[t]\\
credit & MCAR  & 0.10$\pm$0.0 & 0.24$\pm$0.1 & 0.10$\pm$0.0 & 0.12$\pm$0.0 & 0.11$\pm$0.0 & 0.12$\pm$0.0 \\
\end{tabular}%
}
	\label{tab:linear}%
\end{table}%

\subsubsection{Results: Data Valuation Defense}
\label{sec:datavaluation}
We evaluated the efficacy of the LAVA (for classification) \cite{just2023lava} and KNN Shapley (for regression \cite{jia2019efficient} data valuation defenses against our AM attack on the ca-housing (regression) and wine quality (classification) data sets. In both cases the modeler uses mean imputation and the attacker (BLAMM) uses the CCA attack. A defense is successful if it results in an average $p$-value less than $0.05$, meaning the modeler rejects the null hypothesis that $\vtheta_t=0$. We used the OpenDataVal \cite{jiang2023opendataval} implementations of both data valuation methods.

The KNN Shapley defense was successful on the ca-housing data set for the CCA attack after 30\% of the imputed data was discarded (Figure~\ref{fig:dataval}, left bottom). On the wine-quality data set, although LAVA reduced the average $p$-value to close to 0.1 after discarding 50\% of the imputed data, the defense was unsuccessful in reducing the $p$-value below the significance threshold (Figure~\ref{fig:dataval}, right bottom). More importantly, despite the data valuation defense, the $\ell_1$ distance of the estimated coefficients to the \adj parameter is 6 to 7 times smaller than the distance to the true coefficients estimated using the complete set $\mZ$ (Figure~\ref{fig:dataval}, left top).

\begin{figure}[t]
	\begin{center}
	\includegraphics[width=0.7\textwidth]{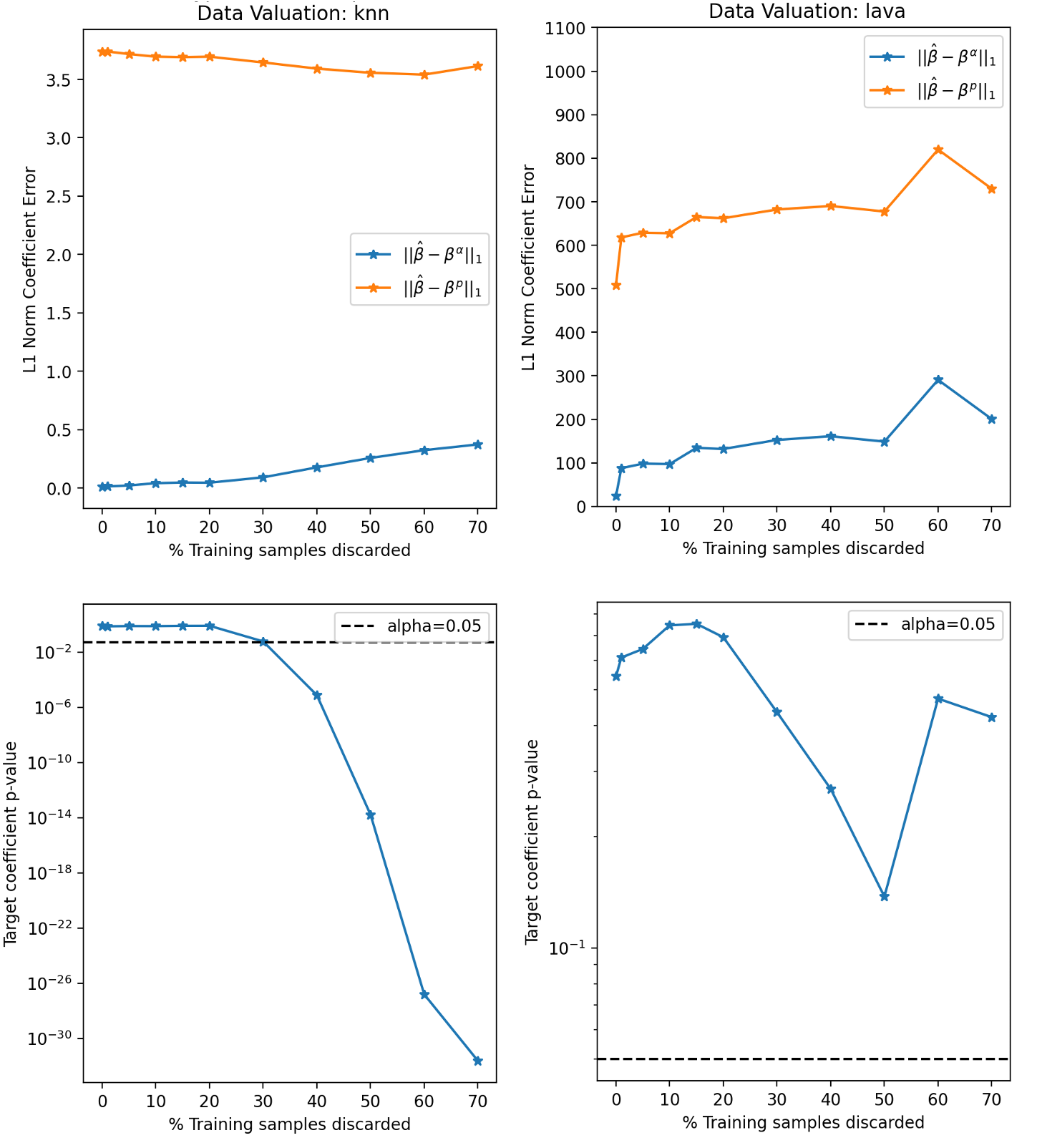}
	\end{center}
	%	\vspace{.15in}
	\caption{Results when KNN-Shapley (left) and LAVA (right) data valuation defenses are used. The top panels show the average (over 20 trials) $\ell_1$ distances between the coefficient estimated in the poisoned dataset and the adversarial coefficients (blue) and the true coefficients (orange), as a function of the number of samples discarded by the modeler. The bottom panels show the corresponding average $p$-values of the target coefficient on a log-scale. }
	\label{fig:dataval}
\end{figure}
\clearpage

\subsubsection{Results: Diabetes Dataset}

This section presents the results on the Diabetes dataset. 

\begin{table}[h!]
	\centering
	\caption{The average (over 20 trials) $\ell$-1 norm between the modeler estimated coefficient and \adj coefficient, i.e. $||\that-\ta||_1$ and its standard deviation (denoted as $\pm$). Lower value indicates attacker is closer to success. If the target coefficient resulting from the attack is insignificant on average (i.e. average p-value $>$ 0.05, Table \ref{tab:pvalue}), it is indicated as a $\checkmark$. The first two columns specifies the modeler's type and attack type respectively. The thrid column corresponds to the result for the \mm and the fourth one baseline MCAR. See Table \ref{tab:metric} for additional performance measures. Because the target variable for the diabetes dataset is discrete, only the CCA attack was applicable to that dataset. }
	% Table generated by Excel2LaTeX from sheet 'dist_a_diabeter'
\begin{tabular}{|l|l|ll|}
\hline
Modeler & Attack & \multicolumn{2}{c|}{diabetes} \bigstrut\\
\cline{3-4}Type  & Type  & BLAMM  & MCAR \bigstrut\\
\hline
\hline
cca   & cca   & \textbf{8.8$\pm$4 (\checkmark)} & 33.9$\pm$6 \bigstrut\\
\hline
mice  & cca   & \textbf{5.0$\pm$4 (\checkmark)} & 30.9$\pm$5 \bigstrut\\
\hline
\end{tabular}%

	\label{tab:norm_diabetes}%
\end{table}%
\clearpage
\subsection{Results: Additional Performance Metrics}
\begin{table}[h!]
	\centering
	\caption{The average (over 20 trials) p-values of the estimated target coefficient $\that_{t}$. Higher value indicates attacker is closer to success. The first two columns specifies the modeler's type and attack type respectively. Remaining columns correspond to four different datasets. In each data set The first column corresponds to the result for the \mm and the second one baseline MCAR. Mean imputation is not applicable to the $diabetes$.}
	\label{tab:pvalue}
	\resizebox{1\columnwidth}{!}{%
		% Table generated by Excel2LaTeX from sheet 'pval'
\begin{tabular}{|c|l|rr|rr|rr|ll|}
\hline
\multicolumn{1}{|l|}{Modeler} & Attack & \multicolumn{2}{c|}{ca-housing} & \multicolumn{2}{c|}{wine-quality} & \multicolumn{2}{c|}{german-credit} & \multicolumn{2}{c|}{diabetes} \bigstrut\\
\cline{3-10}\multicolumn{1}{|l|}{Type} & Type  & \multicolumn{1}{l}{BLAMM} & \multicolumn{1}{l|}{MCAR} & \multicolumn{1}{l}{BLAMM} & \multicolumn{1}{l|}{MCAR} & \multicolumn{1}{l}{BLAMM} & \multicolumn{1}{l|}{MCAR} & BLAMM  & MCAR \bigstrut\\
\hline
\hline
\multirow{3}[6]{*}{mean} & mean  & 0.791 & 0     & 0.811 & 9.00E-32 & 0.869 & 6.00E-04 & -     & - \bigstrut\\
\cline{2-10}      & cca   & 0.749 & 0     & 0.543 & 2.00E-22 & 0.731 & 9.00E-04 & -     & - \bigstrut\\
\cline{2-10}      & linear & 0.093 & 0     & 0.73  & 4.00E-22 & 0.73  & 9.00E-04 & -     & - \bigstrut\\
\hline
\multirow{3}[6]{*}{cca} & mean  & 0     & 0     & 5.00E-22 & 2.00E-40 & 0.488 & 5.00E-04 & -     & - \bigstrut\\
\cline{2-10}      & cca   & 0.099 & 0     & 0.641 & 2.00E-30 & 0.789 & 6.00E-04 & \multicolumn{1}{r}{0.83} & \multicolumn{1}{r|}{0.004} \bigstrut\\
\cline{2-10}      & linear & 0.785 & 0     & 0.604 & 2.00E-30 & 0.819 & 6.00E-04 & -     & - \bigstrut\\
\hline
\multirow{3}[6]{*}{linear} & mean  & 0     & 0     & 1.00E-30 & 4.00E-43 & 0.424 & 3.00E-04 & -     & - \bigstrut\\
\cline{2-10}      & cca   & 4.00E-04 & 0     & 0.59  & 2.00E-40 & 0.782 & 2.00E-04 & -     & - \bigstrut\\
\cline{2-10}      & linear & 0.559 & 0     & 0.558 & 1.00E-40 & 0.809 & 2.00E-04 & -     & - \bigstrut\\
\hline
\multirow{3}[6]{*}{mice} & mean  & 0     & 0     & 7.00E-16 & 3.00E-45 & 0.407 & 6.00E-04 & -     & - \bigstrut\\
\cline{2-10}      & cca   & 8.00E-24 & 0     & 0.003 & 1.00E-43 & 0.698 & 9.00E-04 & \multicolumn{1}{r}{0.609} & \multicolumn{1}{r|}{0.003} \bigstrut\\
\cline{2-10}      & linear & 4.00E-11 & 0     & 1.00E-03 & 5.00E-46 & 0.657 & 6.00E-04 & -     & - \bigstrut\\
\hline
\end{tabular}%
	
	}
\end{table}%

\begin{table}[h!]
	\centering
	\caption{The average (over 20 trials) test set performance (NMSE or Accuracy) of the modeler estimated GLM under missing data attacks. In the regression datasets ($ca-housing$, $diabetes$) entries with lower NMSE and in classification datasets ($wine-quality$,$german-credit$) entries with higher accuracy are marked. The first two columns specifies the modeler's type and attack type respectively. Remaining columns correspond to four different datasets. In each data set The first column corresponds to the result for the \mm and the second one baseline MCAR. Mean imputation is not applicable to the $diabetes$.}
	\label{tab:metric}
	\resizebox{1\columnwidth}{!}{%
		% Table generated by Excel2LaTeX from sheet 'metric'
\begin{tabular}{|c|l|ll|ll|ll|ll|}
\hline
\multicolumn{1}{|l|}{Modeler} & Attack & \multicolumn{2}{c|}{ca-housing} & \multicolumn{2}{c|}{wine-quality} & \multicolumn{2}{c|}{german-credit} & \multicolumn{2}{c|}{diabetes} \bigstrut\\
\cline{3-10}\multicolumn{1}{|l|}{Type} & Type  & BLAMM  & MCAR  & BLAMM  & MCAR  & BLAMM  & MCAR  & BLAMM  & MCAR \bigstrut\\
\hline
\hline
\multirow{3}[6]{*}{mean} & mean  & 0.59$\pm$0.0 & 0.42$\pm$0.0 & 0.721$\pm$0.0 & 0.721$\pm$0.0 & 0.78$\pm$0.0 & 0.796$\pm$0.0 & 0.78$\pm$0.0 & 0.796$\pm$0.0 \bigstrut\\
\cline{2-2}      & cca   & 0.59$\pm$0.0 & 0.45$\pm$0.0 & 0.72$\pm$0.0 & 0.718$\pm$0.0 & 0.78$\pm$0.0 & 0.791$\pm$0.0 & 0.78$\pm$0.0 & 0.791$\pm$0.0 \bigstrut\\
\cline{2-2}      & linear & 0.61$\pm$0.0 & 0.45$\pm$0.0 & 0.721$\pm$0.0 & 0.718$\pm$0.0 & 0.782$\pm$0.01 & 0.791$\pm$0.0 & 0.782$\pm$0.01 & 0.791$\pm$0.0 \bigstrut\\
\cline{1-2}\multirow{3}[6]{*}{cca} & mean  & 0.49$\pm$0.0 & 0.42$\pm$0.0 & 0.737$\pm$0.0 & 0.725$\pm$0.0 & 0.768$\pm$0.01 & 0.792$\pm$0.01 & 0.768$\pm$0.01 & 0.792$\pm$0.01 \bigstrut\\
\cline{2-2}      & cca   & 0.58$\pm$0.0 & 0.42$\pm$0.0 & 0.719$\pm$0.0 & 0.726$\pm$0.0 & 0.781$\pm$0.01 & 0.794$\pm$0.01 & 0.56$\pm$0.0 & 0.55$\pm$0.0 \bigstrut\\
\cline{2-2}      & linear & 0.59$\pm$0.0 & 0.42$\pm$0.0 & 0.718$\pm$0.0 & 0.726$\pm$0.0 & 0.779$\pm$0.01 & 0.794$\pm$0.01 & 0.779$\pm$0.01 & 0.794$\pm$0.01 \bigstrut\\
\cline{1-2}\multirow{3}[6]{*}{linear} & mean  & 0.51$\pm$0.0 & 0.44$\pm$0.0 & 0.721$\pm$0.0 & 0.725$\pm$0.0 & 0.79$\pm$0.0 & 0.796$\pm$0.0 & 0.79$\pm$0.0 & 0.796$\pm$0.0 \bigstrut\\
\cline{2-2}      & cca   & 0.54$\pm$0.0 & 0.48$\pm$0.0 & 0.72$\pm$0.0 & 0.728$\pm$0.0 & 0.778$\pm$0.0 & 0.794$\pm$0.0 & 0.778$\pm$0.0 & 0.794$\pm$0.0 \bigstrut\\
\cline{2-2}      & linear & 0.59$\pm$0.0 & 0.48$\pm$0.0 & 0.72$\pm$0.0 & 0.728$\pm$0.0 & 0.778$\pm$0.0 & 0.794$\pm$0.0 & 0.778$\pm$0.0 & 0.794$\pm$0.0 \bigstrut\\
\cline{1-2}\multirow{3}[6]{*}{mice} & mean  & 0.44$\pm$0.0 & 0.42$\pm$0.0 & 0.717$\pm$0.0 & 0.723$\pm$0.0 & 0.787$\pm$0.0 & 0.795$\pm$0.01 & 0.787$\pm$0.0 & 0.795$\pm$0.01 \bigstrut\\
\cline{2-2}      & cca   & 0.5$\pm$0.0 & 0.42$\pm$0.0 & 0.722$\pm$0.0 & 0.721$\pm$0.0 & 0.781$\pm$0.01 & 0.794$\pm$0.0 & 0.56$\pm$0.0 & 0.54$\pm$0.0 \bigstrut\\
\cline{2-2}      & linear & 0.52$\pm$0.0 & 0.42$\pm$0.0 & 0.722$\pm$0.0 & 0.721$\pm$0.0 & 0.781$\pm$0.01 & 0.796$\pm$0.01 & 0.781$\pm$0.01 & 0.796$\pm$0.01 \bigstrut\\
\cline{1-2}\end{tabular}%
	
	}
\end{table}%
\clearpage
\subsection*{Dataset specific results: CA-Housing}

\begin{figure}[h!]
	\centering
	\includegraphics[width=0.5\textwidth]{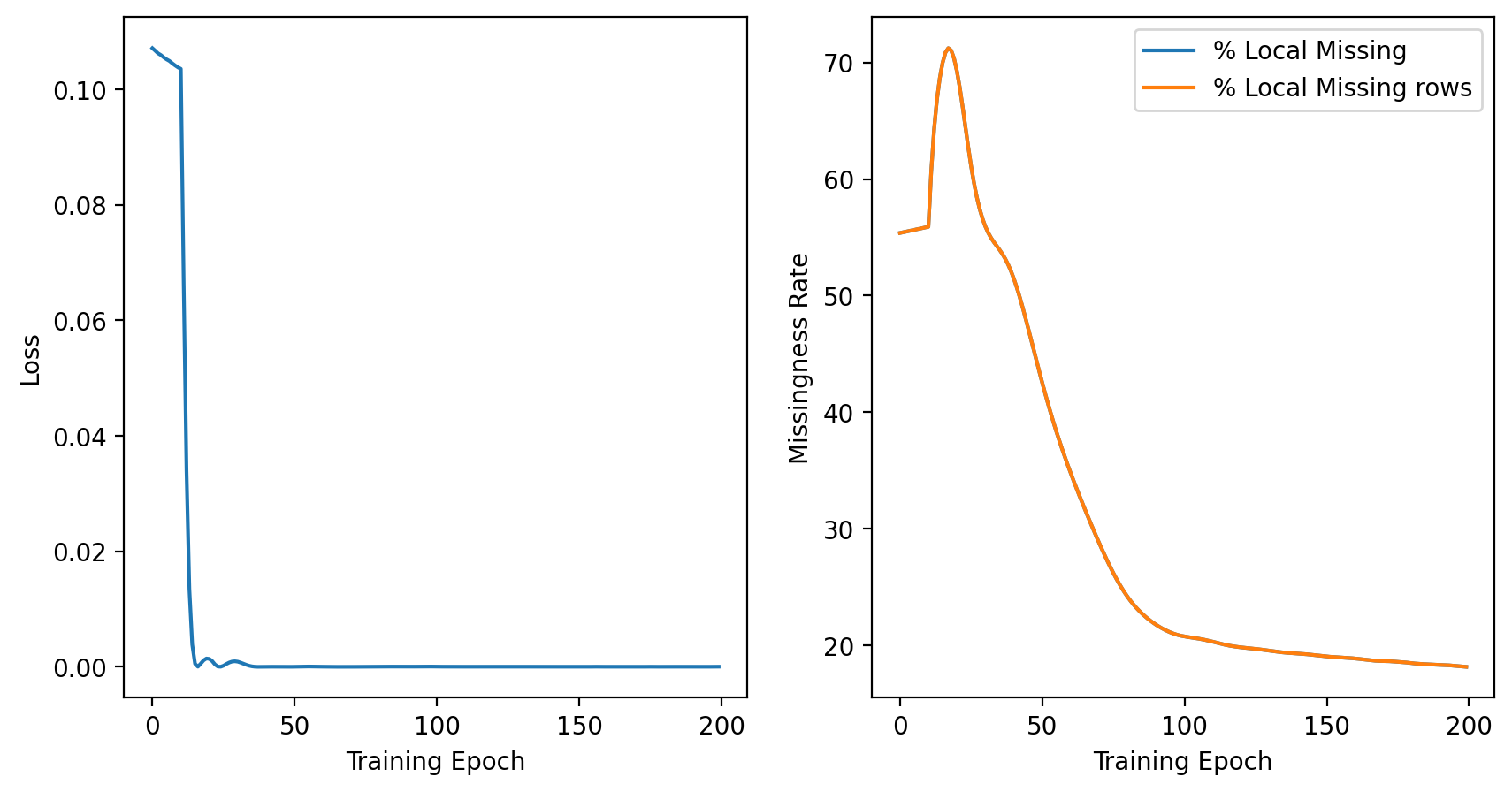}
	\caption{Training curve of the neural network for mean imputation attack. Left: Loss, Right: \% missingness of the target variable.}
	\label{fig:cali}
\end{figure}
\begin{table}[h]
	\centering
	\caption{The GLM coefficients and p-values in the $ca-housing$ dataset. Each row corresponds to a different feature (including bias term). The target variable is denoted in bold. The columns $2$ to $5$ corresponds to coefficients and $6$ to $9$ corresponds to their respective p-values of the coefficients. Columns 2 and 3 correspond to the complete data estimate $\tr$ and the adversarial parameter $\ta$. Columns 4 and 5 correspond to average modeler estimated $\that$ in BLAMM and MCAR missigness respectively under the mean imputation attack for the mean imputation modeler (out of 20 trials).}
	\resizebox{1\columnwidth}{!}{%
		% Table generated by Excel2LaTeX from sheet 'cali mean'
\begin{tabular}{|l|l|l|l|l||l|l|l|l|}
\hline
\multicolumn{1}{|c|}{\multirow{2}[4]{*}{Features}} & \multicolumn{1}{c|}{\multirow{2}[4]{*}{$\tr$}} & \multicolumn{1}{c|}{\multirow{2}[4]{*}{$\ta$}} & \multicolumn{2}{c|}{Attack \& Modeler = Mean} & \multicolumn{1}{c|}{\multirow{2}[4]{*}{p-val.\newline{}($\tr$)}} & \multicolumn{1}{c|}{\multirow{2}[4]{*}{p-val.\newline{}($\ta$)}} & \multicolumn{2}{c|}{Attack \& Modeler = Mean} \bigstrut\\
\cline{4-5}\cline{8-9}      &       &       & \multicolumn{1}{c|}{$\that$ (BLAMM)} & \multicolumn{1}{c|}{$\that$ (MCAR)} &       &       & \multicolumn{1}{c|}{p-val. (BLAMM)} & \multicolumn{1}{c|}{p-val. (MCAR)} \bigstrut\\
\hline
\hline
\textbf{MedInc} & 0.45  & 0.0   & 0.0$\pm$0.0 & 0.35$\pm$0.0 & 0.0   & -     & 0.791 & 0.0 \bigstrut[t]\\
HouseAge & 0.01  & 0.0   & 0.0$\pm$0.0 & 0.01$\pm$0.0 & 4E-85 & 3E-13 & 4E-13 & 6E-46 \\
AveRooms & -0.12 & 0.35  & 0.35$\pm$0.0 & 0.05$\pm$0.0 & 5E-77 & 0.0   & 0.0   & 2E-10 \\
AveBedrms & 0.78  & -1.43 & -1.42$\pm$0.0 & -0.01$\pm$0.0 & 1E-120 & 0.0   & 0.0   & 0.563 \\
Population & -0.0  & -0.0  & -0.0$\pm$0.0 & -0.0$\pm$0.0 & 0.699 & 0.019 & 0.02  & 0.262 \\
AveOccup & -0.0  & -0.0  & -0.0$\pm$0.0 & -0.0$\pm$0.0 & 4E-13 & 0.018 & 0.018 & 3E-04 \\
Latitude & -0.42 & -0.73 & -0.72$\pm$0.0 & -0.53$\pm$0.0 & 0.0   & 0.0   & 0.0   & 0.0 \\
Longitude & -0.43 & -0.72 & -0.72$\pm$0.0 & -0.54$\pm$0.0 & 0.0   & 0.0   & 0.0   & 0.0 \\
bias  & -37.02 & -58.62 & -58.59$\pm$0.0 & -44.93$\pm$0.3 & 0.0   & 0.0   & 0.0   & 0.0 \bigstrut[b]\\
\hline
\end{tabular}%

	}
	\label{tab:ca}%
\end{table}%
\clearpage
\subsection*{Dataset specific results: Wine-quality}

\begin{figure}[h!]
	\centering
	\includegraphics[width=0.5\textwidth]{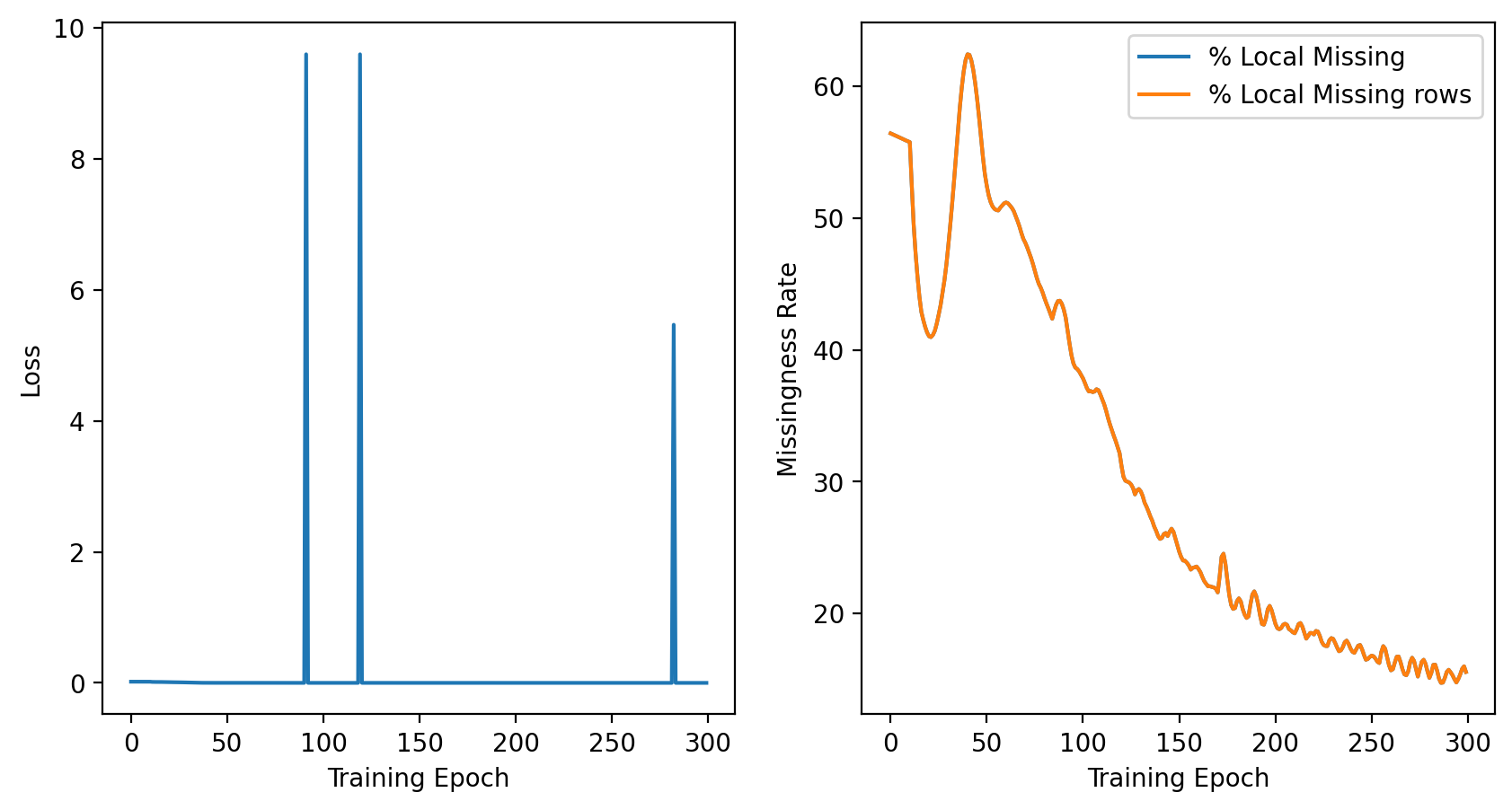}
	\caption{Training curve of the neural network for CCA attack. Left: Loss, Right: \% missingness of the target variable. In our threat model we do assume standardized features and in this problem it created slight numerical instability during training.}
	\label{fig:wine}
\end{figure}
\begin{table}[h!]
	\centering
	\caption{The GLM coefficients and p-values in the $wine-quality$ dataset. Each row corresponds to a different feature (including bias term). The target variable is denoted in bold. The columns $2$ to $5$ corresponds to coefficients and $6$ to $9$ corresponds to their respective p-values of the coefficients. Columns 2 and 3 correspond to the complete data estimate $\tr$ and the adversarial parameter $\ta$. Columns 4 and 5 correspond to average modeler estimated $\that$ in BLAMM and MCAR missigness respectively under the CCA attack for the CCA modeler (out of 20 trials).}
	\resizebox{1\columnwidth}{!}{%
		% Table generated by Excel2LaTeX from sheet 'wine-quality'
\begin{tabular}{|l|l|l|l|l|l|l|l|l|}
\hline
\multicolumn{1}{|c|}{\multirow{2}[4]{*}{Features}} & \multicolumn{1}{c|}{\multirow{2}[4]{*}{$\tr$}} & \multicolumn{1}{c|}{\multirow{2}[4]{*}{$\ta$}} & \multicolumn{2}{c|}{Attack \& Modeler = CCA} & \multicolumn{1}{c|}{\multirow{2}[4]{*}{p-val.\newline{}($\tr$)}} & \multicolumn{1}{c|}{\multirow{2}[4]{*}{p-val.\newline{}($\ta$)}} & \multicolumn{2}{c|}{Attack \& Modeler = CCA} \bigstrut\\
\cline{4-5}\cline{8-9}      &       &       & \multicolumn{1}{c|}{$\that$ (BLAMM)} & \multicolumn{1}{c|}{$\that$ (mcar)} &       &       & \multicolumn{1}{c|}{p-val. (BLAMM)} & \multicolumn{1}{c|}{p-val. (MCAR)} \bigstrut\\
\hline
\hline
fixed\_acidity & 0.12  & 0.62  & 0.6$\pm$0.0 & 0.12$\pm$0.0 & 0.036 & 2E-41 & 4E-16 & 0.068 \bigstrut[t]\\
volatile\_acidity & -4.69 & -3.35 & -3.46$\pm$0.1 & -4.67$\pm$0.1 & 1E-49 & 3E-32 & 1E-22 & 9E-40 \\
citric\_acid & -0.77 & -0.47 & -0.48$\pm$0.1 & -0.76$\pm$0.2 & 0.007 & 0.091 & 0.128 & 0.032 \\
residual\_sugar & 0.1   & 0.27  & 0.26$\pm$0.0 & 0.1$\pm$0.0 & 7E-08 & 9E-77 & 1E-27 & 4E-06 \\
chlorides & 0.26  & 0.1   & 0.29$\pm$0.7 & 0.23$\pm$0.6 & 0.823 & 0.931 & 0.715 & 0.75 \\
\multicolumn{1}{|p{3.75em}|}{free SO2} & 0.02  & 0.02  & 0.02$\pm$0.0 & 0.02$\pm$0.0 & 2E-10 & 9E-12 & 3E-10 & 8E-09 \\
\multicolumn{1}{|p{3.75em}|}{total SO2} & -0.01 & -0.01 & -0.01$\pm$0.0 & -0.01$\pm$0.0 & 7E-16 & 3E-34 & 1E-24 & 5E-13 \\
density & -75.9 & -599.63 & -581.77$\pm$19.6 & -79.2$\pm$21.1 & 0.091 & 3E-108 & 1E-20 & 0.136 \\
pH    & 0.89  & 3.43  & 3.26$\pm$0.1 & 0.89$\pm$0.2 & 0.008 & 2E-33 & 1E-15 & 0.025 \\
sulphates & 2.2   & 3.4   & 3.43$\pm$0.1 & 2.19$\pm$0.1 & 2E-13 & 7E-32 & 2E-20 & 1E-10 \\
\textbf{alcohol} & 0.9   & 0.0   & 0.04$\pm$0.0 & 0.89$\pm$0.0 & 4E-43 & -     & 0.641 & 2E-30 \\
bias  & 63.51 & 580.4 & 562.97$\pm$19.2 & 66.87$\pm$20.6 & 0.149 & 2E-108 & 4E-20 & 0.198 \bigstrut[b]\\
\hline
\end{tabular}%

	}
	\label{tab:wine}%
\end{table}%

\clearpage
\subsection*{Dataset specific results: Diabetes}

\begin{figure}[h!]
	\centering
	\includegraphics[width=0.5\textwidth]{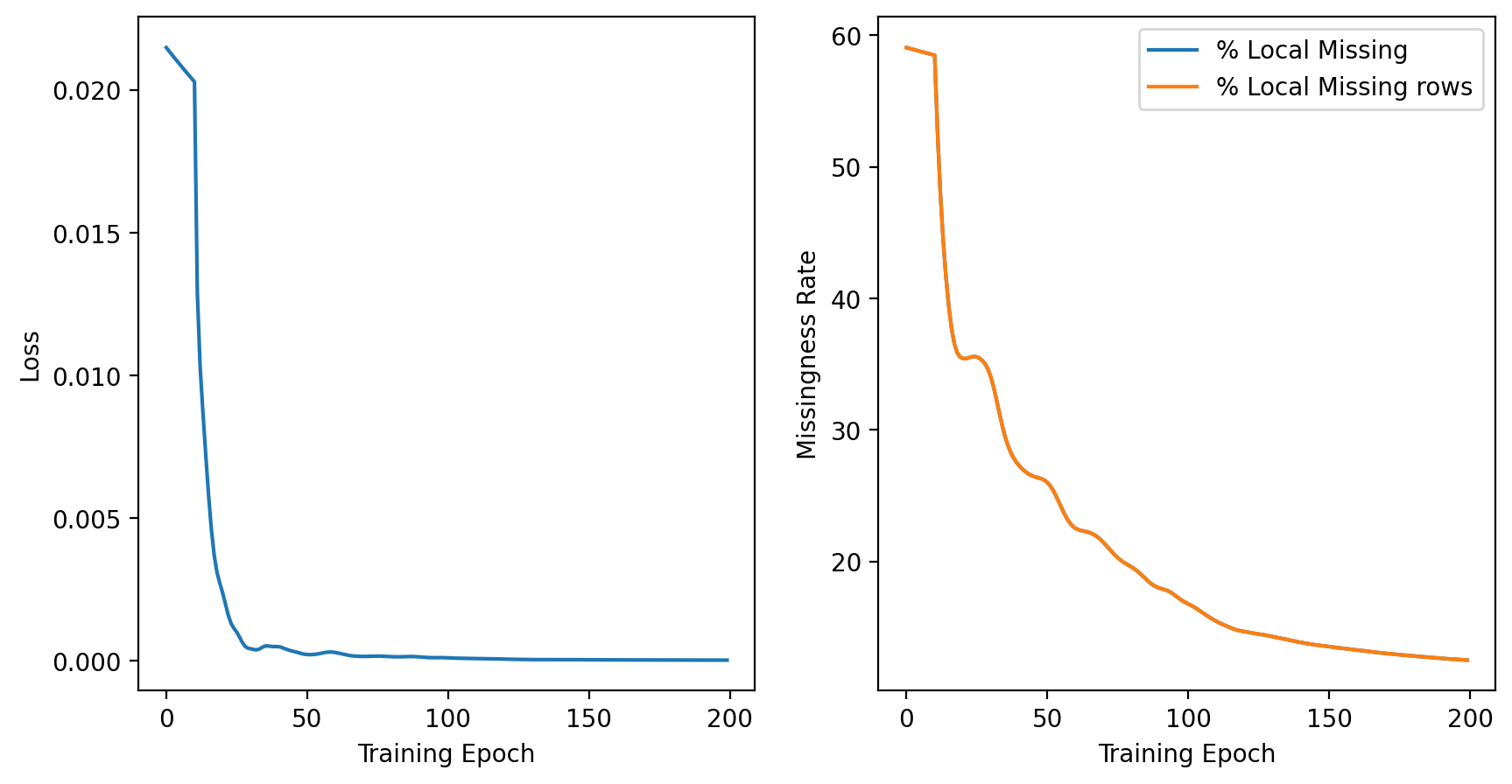}
	\caption{Training curve of the neural network for CCA attack. Left: Loss, Right: \% missingness of the target variable.}
	\label{fig:diabetes}
\end{figure}
\begin{table}[h!]
	\centering
	\caption{The GLM coefficients and p-values in the $diabetes$ dataset. Each row corresponds to a different feature (including bias term). The target variable is denoted in bold. The columns $2$ to $5$ corresponds to coefficients and $6$ to $9$ corresponds to their respective p-values of the coefficients. Columns 2 and 3 correspond to the complete data estimate $\tr$ and the adversarial parameter $\ta$. Columns 4 and 5 correspond to average modeler estimated $\that$ in BLAMM and MCAR missigness respectively under the CCA attack for the CCA modeler (out of 20 trials).}
	\resizebox{1\columnwidth}{!}{%
		% Table generated by Excel2LaTeX from sheet 'diabetes'
\begin{tabular}{|l|l|l|l|l|l|l|l|l|}
\hline
\multicolumn{1}{|c|}{\multirow{2}[4]{*}{Features}} & \multicolumn{1}{c|}{\multirow{2}[4]{*}{$\tr$}} & \multicolumn{1}{c|}{\multirow{2}[4]{*}{$\ta$}} & \multicolumn{2}{c|}{Attack \& Modeler = CCA} & \multicolumn{1}{c|}{\multirow{2}[4]{*}{p-val.\newline{}($\tr$)}} & \multicolumn{1}{c|}{\multirow{2}[4]{*}{p-val.\newline{}($\ta$)}} & \multicolumn{2}{c|}{Attack \& Modeler = CCA} \bigstrut\\
\cline{4-5}\cline{8-9}      &       &       & \multicolumn{1}{c|}{$\that$ (BLAMM)} & \multicolumn{1}{c|}{$\that$ (mcar)} &       &       & \multicolumn{1}{c|}{p-val. (BLAMM)} & \multicolumn{1}{c|}{p-val. (MCAR)} \bigstrut\\
\hline
\hline
AGE   & 0.14  & 0.01  & 0.04$\pm$0.1 & 0.14$\pm$0.1 & 0.583 & 0.96  & 0.792 & 0.6 \bigstrut[t]\\
BMI   & 5.85  & 6.36  & 6.42$\pm$0.2 & 5.86$\pm$0.2 & 1E-11 & 2E-13 & 1E-12 & 6E-10 \\
BP    & 1.2   & 1.05  & 1.04$\pm$0.1 & 1.2$\pm$0.1 & 2E-06 & 3E-05 & 9E-05 & 2E-05 \\
S1    & -1.28 & -1.22 & -1.14$\pm$0.2 & -1.29$\pm$0.2 & 0.04  & 0.054 & 0.073 & 0.064 \\
S2    & 0.81  & 0.8   & 0.73$\pm$0.1 & 0.8$\pm$0.2 & 0.156 & 0.169 & 0.212 & 0.208 \\
S3    & 0.6   & 0.72  & 0.65$\pm$0.2 & 0.64$\pm$0.3 & 0.484 & 0.411 & 0.455 & 0.509 \\
S4    & 10.16 & 6.77  & 6.02$\pm$1.9 & 11.0$\pm$3.1 & 0.138 & 0.326 & 0.407 & 0.168 \\
S5    & 67.11 & 69.53 & 67.89$\pm$5.4 & 66.71$\pm$5.1 & 2E-04 & 1E-04 & 2E-04 & 7E-04 \\
S6    & 0.2   & 0.16  & 0.18$\pm$0.1 & 0.21$\pm$0.1 & 0.508 & 0.61  & 0.58  & 0.536 \\
\textbf{SEX} & -23.06 & 0.0   & 0.52$\pm$2.2 & -22.97$\pm$2.8 & 5E-04 & -     & 0.83  & 0.004 \\
bias  & -364.44 & -378.54 & -375.39$\pm$23.4 & -367.4$\pm$32.2 & 1E-06 & 6E-07 & 1E-06 & 2E-05 \bigstrut[b]\\
\hline
\end{tabular}%

	}
	\label{tab:diabetes}%
\end{table}%

\subsection{Manipulating ATE under Partial data access}

\subsubsection{Setup: BLAMM and Hyper-parameters}

We implemented BLAMM using tensorflow by defining a custom keras model. We initialized the data-specific weights $\phi^{(i)}$ to zero. We used Adam optimizer with learning rate 0.05 and optimized the $\phi$ for 300 epochs (max\_steps). For $N_0/N=\{1,0.75,0.5,0.25\}$, we respectively used ${\lambda=\{0.5,0.25,1e-4,1e-4\}}$. We observed numerical instability in the lower-level problem as the access proportion decreased. To overcome this, we adjusted the BLAMM algorithm such that when the lower-level problem's solution was practically infeasible, the lower-level problem is solved again with an added $\ell-2$ regularization term weighted by 1e-3. A solution was deemed practically infeasible if the magnitude of the estimated logistic regression model's intercept is higher than $1e10$. We applied this adjustment when $N_0/N\neq1$.

We used a Intel(R) Core(TM) i7-6700K CPU @ 4.00GHz with 31Gi of RAM to run our experiments. Operating system is Ubuntu 20. The longest training took less than 20 minutes on the CPU.

\subsubsection{Setup: Modeler Parameters}

We used scikit-learn \cite{scikit-learn} for logistic regression with hyper-parameters solver=lbfgs, max\_iter=600 and penalty='none'. We used 
CATENets\cite{curthReallyDoingGreat2021} (v0.2.4) for TARnet and Tnet. We set the hyper-parameters to the values used in \cite{curthInductiveBiasesHeterogeneous2021} (as reported in their github  \footnote{\url{https://github.com/AliciaCurth/CATENets/blob/main/experiments/experiments_inductivebias_NeurIPS21/experiments_twins.py}}). We used R package ``grf'' (v2.4.0) \cite{atheyGeneralizedRandomForests2019} for the implementation of Causal Forest following  \cite{curthReallyDoingGreat2021}\footnote{\url{https://github.com/AliciaCurth/CATENets/blob/main/experiments/experiments_benchmarks_NeurIPS21/twins_experiments_grf.R}} and used the default hyper-parameters.

\subsubsection{Results: Convergence Curves}

This subsection presents the convergence curves of the BLAMM Algorithm.

	\begin{figure}[ h!]
		\centering
		\begin{subfigure}[b]{0.45\textwidth}
			\centering
			\includegraphics[width=\textwidth]{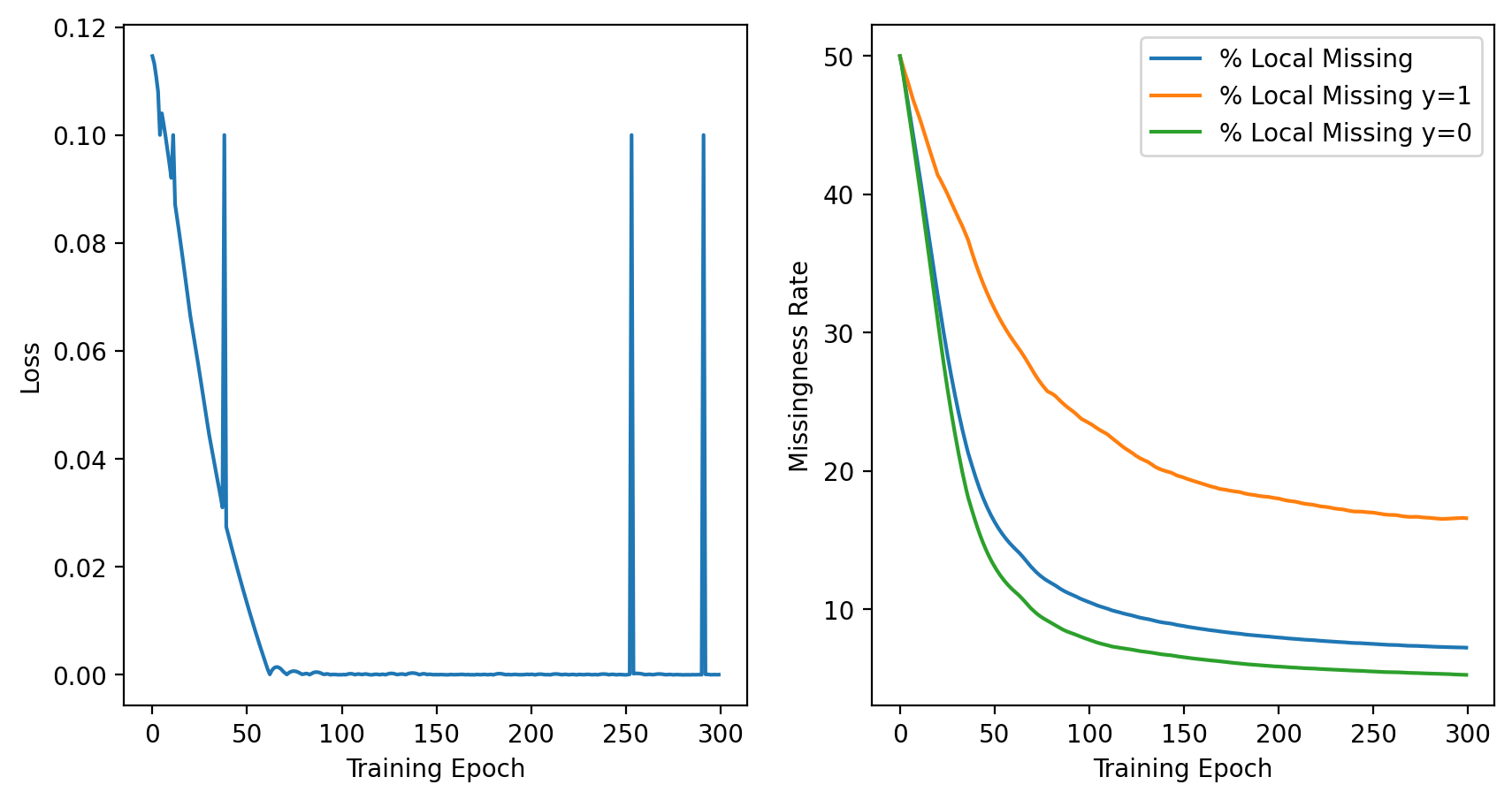}
			\caption{$N_0/N=1$}
			\label{fig:fig1}
		\end{subfigure}
		\hfill
		\begin{subfigure}[b]{0.45\textwidth}
			\centering
			\includegraphics[width=\textwidth]{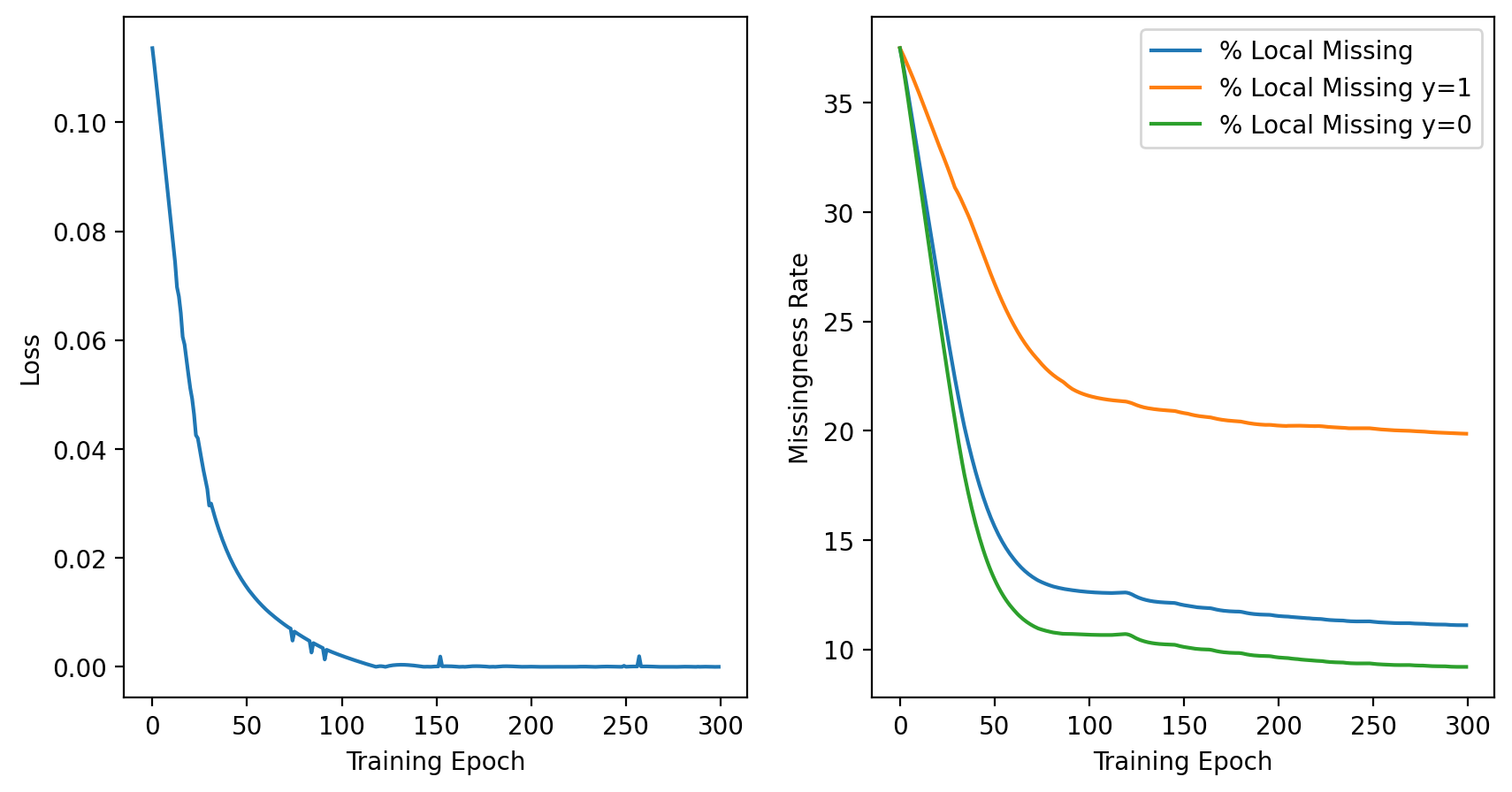}
			\caption{$N_0/N=.75$}
			\label{fig:fig2}
		\end{subfigure}
		
		\vspace{0.5cm} % Space between rows
		
		\begin{subfigure}[b]{0.45\textwidth}
			\centering
			\includegraphics[width=\textwidth]{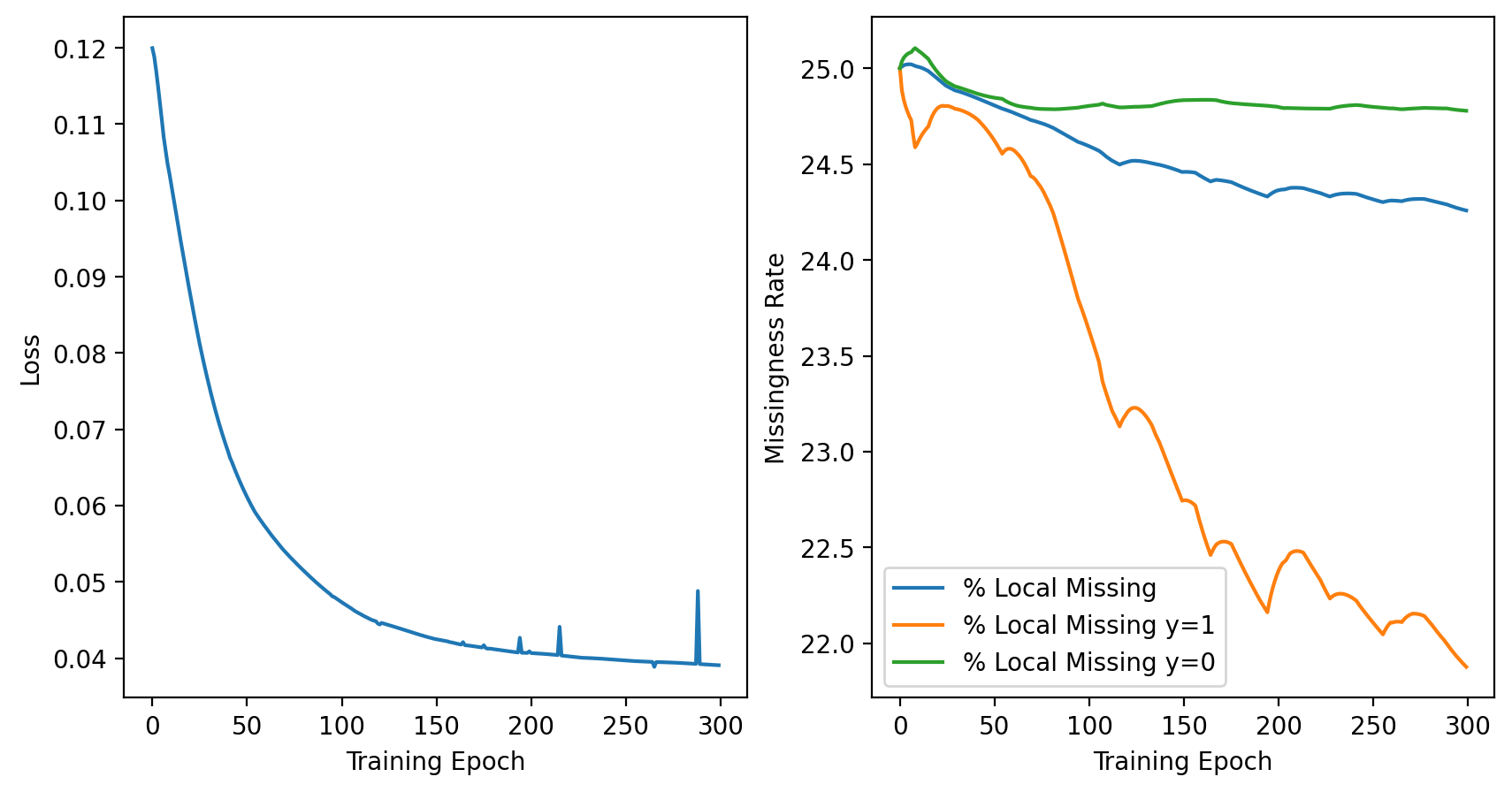}
			\caption{$N_0/N=.5$}
			\label{fig:fig3}
		\end{subfigure}
		\hfill
		\begin{subfigure}[b]{0.45\textwidth}
			\centering
			\includegraphics[width=\textwidth]{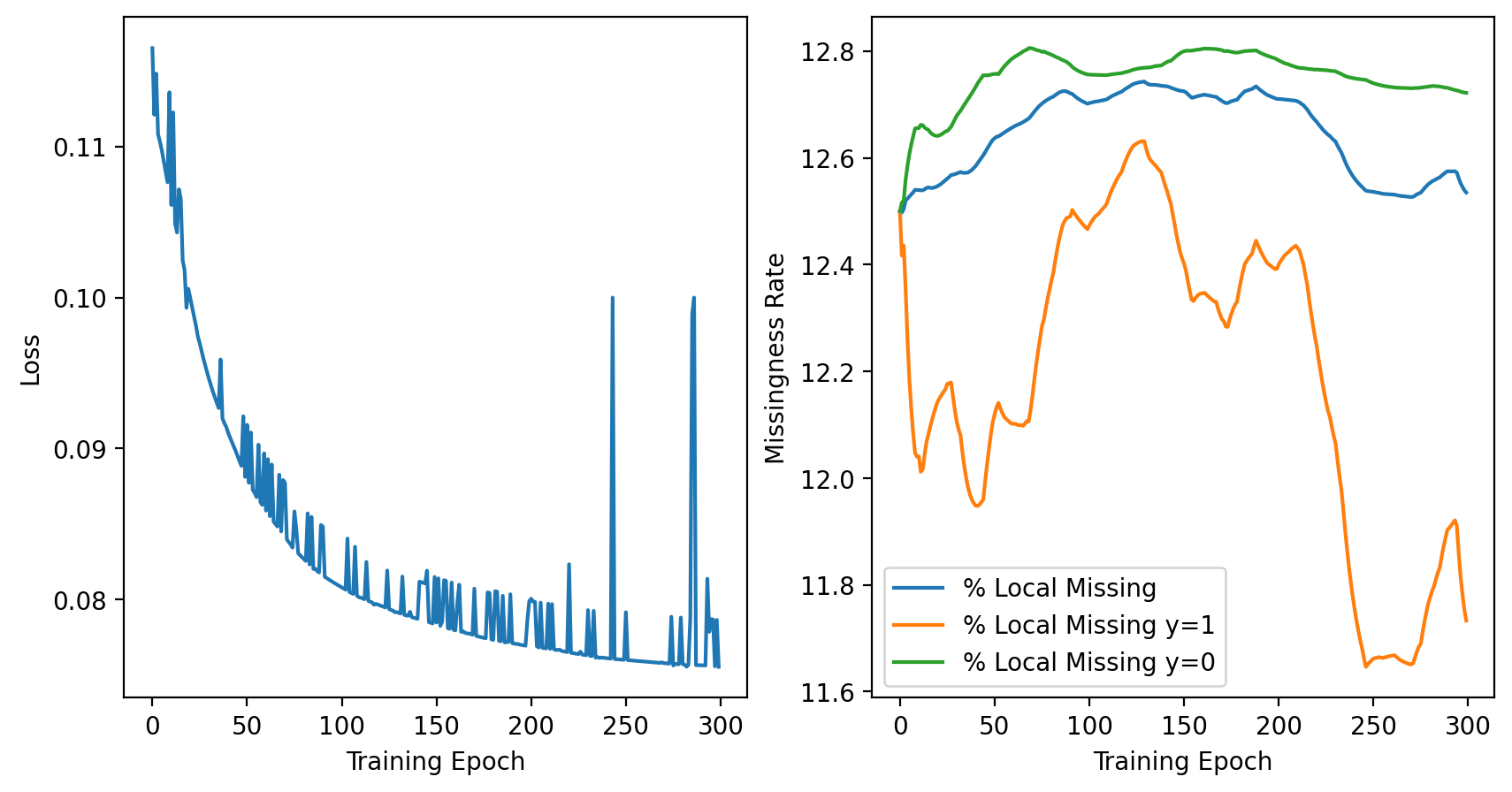}
			\caption{$N_0/N=.25$}
			\label{fig:fig4}
		\end{subfigure}
		
		\caption{Converge of the BLAMM Algorithm under different $N_0/N$ values. In each subfigure there are two figures. The left one denotes the upper-level loss function without the regularization term. The right one denotes the expected local missingness amount in the two masked variables (blue) and the expected local missingness amount in the two masked variables stratified by the outcome (green when $y=0$, and orange when $y=1$). Notice the axes limits are independently set.}
	\end{figure}

%%%%%%%%%%%%%%%%%%%%%%%%%%%%%%%%%%%%%%%%%%%%%%%%%%%%%%%%%%%%

%\newpage
\clearpage
\subsubsection{Results on Additional Missing Data Remediation Methods}

We have tested the BLAMM attack, trained using mean imputation (Table \ref{tab:ate}), against three new remediation techniques: MICE imputation; CF with MIA splitting and standard ATE estimation; CF with MIA splitting and doubly robust estimation \cite{mayerDoublyRobustTreatment2020}.

The results show that our BLAMM attack successfully inflates the ATE against all three tested methods, though its effectiveness varies.

MICE imputation was less susceptible to the attack than simple mean imputation. Doubly robust estimation was the most effective at mitigating the attack when the adversary had access to a high proportion of data (100\% and 75\%). However, its effectiveness decreased significantly as the access proportion dropped to 50\% and 25\%, a trend not observed with the other methods. This suggests that a higher degree of missingness may be particularly impactful on the propensity scores used in doubly robust estimation.

% MICE imputation and a doubly robust estimation procedure introduced in \cite{mayerDoublyRobustTreatment2020} introduced specifically for handling missing data in covariates. Specifically, this procedure uses the CF algorithm within the Augmented Inverse Probability Weighting (AIPW) estimator and instead of imputing the missing value handles them Missing Incorporated in Attributes (MIA) splitting.  
 
\begin{table}[h!]
	\centering
	\caption{BLAMM attack trained for mean imputation transferred to MICE imputation and CF algorithm variants with MIA splitting. The average ATE (\%) of regression estimators, $\hat{\tau}$ (over 5 trials). ``Access:'' indicates the percentage of rows manipulatable. Bold highlights the missingness closer to the target ATE of 10\%. \newline}
	\resizebox{1\columnwidth}{!}{%
\begin{tabular}{l|ll|ll|ll|ll}
\multicolumn{1}{r}{} & \multicolumn{2}{c}{Access: 100\%} & \multicolumn{2}{c}{Access: 75\%} & \multicolumn{2}{c}{Access: 50\%} & \multicolumn{2}{c}{Access: 25\%} \\
\multicolumn{1}{r}{} & BLAMM & \multicolumn{1}{l}{MCAR} & BLAMM & \multicolumn{1}{l}{MCAR} & BLAMM & \multicolumn{1}{l}{MCAR} & BLAMM & MCAR \bigstrut[b]\\
\hline
\hline
TARnet + mean & \multicolumn{1}{c}{\textbf{10.52+-0.9}} & -1.0+-2.2 & \textbf{8.49+-1.7} & -0.38+-1.9 & \textbf{7.26+-2.3} & 0.04+-0.8 & \textbf{1.62+-2.0} & -0.42+-1.8 \bigstrut[t]\\
Tnet + mean & \multicolumn{1}{c}{\textbf{10.42+-0.9}} & -2.6+-3.1 & \textbf{7.44+-1.4} & -1.72+-3.2 & \textbf{3.67+-1.4} & -4.7+-1.2 & \textbf{-1.5+-0.1} & -3.98+-0.9 \\
linear +mean & \multicolumn{1}{c}{\textbf{9.86+-0.1}} & -1.45+-0.2 & \textbf{10.1+-0.0} & -1.56+-0.2 & \textbf{6.29+-0.0} & -1.65+-0.1 & \textbf{2.26+-0.0} & -1.33+-0.2 \\
CF + mean & \multicolumn{1}{c}{\textbf{7.65+-0.5}} & -1.44+-0.2 & \textbf{3.32+-0.1} & -1.45+-0.3 & \textbf{3.01+-0.0} & -1.3+-0.1 & \textbf{1.54+-0.0} & -1.25+-0.1 \\
TARnet + MICE & \textbf{8.27+-2.5} & -1.1+-1.6 & \textbf{8.0+-2.3} & -0.55+-1.9 & \textbf{7.39+-1.4} & -0.08+-1.5 & \textbf{1.9+-1.2} & 0.23+-1.5 \\
Tnet +MICE & \textbf{6.08+-1.6} & -1.86+-1.0 & \textbf{5.84+-1.0} & -1.02+-2.2 & \textbf{2.93+-2.4} & -2.68+-3.3 & \textbf{-0.18+-1.2} & -5.24+-0.5 \\
linear +MICE & \textbf{6.87+-1.3} & -1.42+-0.2 & \textbf{7.24+-0.5} & -1.32+-0.1 & \textbf{5.2+-0.3} & -1.58+-0.1 & \textbf{1.39+-0.2} & -1.14+-0.2 \\
CF + MIA & \textbf{9.0+-0.5} & -1.45+-0.1 & \textbf{3.47+-0.2} & -1.4+-0.2 & \textbf{-0.41+-0.1} & -1.22+-0.2 & \textbf{-0.95+-0.1} & -1.22+-0.1 \\
CF (DR) + MIA & \textbf{1.59+-0.3} & -1.49+-0.2 & \textbf{0.61+-0.2} & -1.56+-0.3 & \textbf{8.79+-0.1} & -1.49+-0.2 & \textbf{2.79+-0.1} & -1.24+-0.1 \\
\end{tabular}%

	}
	\label{app:tab:ate_mice}%
\end{table}%
\clearpage
\subsubsection{Results on Baseline MNAR Mechanism}\label{app:sec:mnar}

We tested an additional baseline missingness mechanism that satisfies the MNAR condition by allowing the masked variable determine its probability. Specifically, the probability one of the masked covariates is observed, depends on gestat and wtgain, the treatment $W$ and the outcome $Y$ through a logistic model \cite{muzellecMissingDataImputation2020}. Specifically let $Z=(X_{\mathcal{M}}, Y, W)$ and $j\in{\mathcal{M}}$, the missing mask of the j'th feature is sampled by:
\[R_{j} \mid Z \sim \text{Bernoulli}(\text{logit}(c_{j} + \beta_{j}^T Z).\]

The weights $\beta_{j}$ are initialized as in \cite{muzellecMissingDataImputation2020} and the bias term $c_{j}$ is selected to ensure missingess rate is equal to the BLAMM attack.

As shown in Tables \ref{app:tab:mnar},\ref{app:tab:mnar2}, our BLAMM method outperforms both MNAR baseline. BLAMM achieves a significant shift in the estimated ATE, while the MNAR mechanism, even with a disproportionately higher missingness rate (7.2\% vs. 50\%), remains ineffective. These results confirm that BLAMM provides a more effective and targeted adversarial attack, highlighting its utility for worst-case performance analysis of missing data methods.

\begin{table}[h!]
	\centering
	\caption{The average ATE (\%) of regression estimators, $\hat{\tau}$ (over 5 trials). ``Access:'' indicates the percentage of rows manipulatable. Bold highlights the missingness closer to the target ATE of 10\%. \newline}
	\resizebox{1\columnwidth}{!}{%
\begin{tabular}{l|ll|ll|ll|ll|}
\multicolumn{1}{r}{} & \multicolumn{2}{c}{Access: 100\%} & \multicolumn{2}{c}{Access: 75\%} & \multicolumn{2}{c}{Access: 50\%} & \multicolumn{2}{c}{Access: 25\%} \\
\multicolumn{1}{r}{} & BLAMM & \multicolumn{1}{l}{MNAR} & BLAMM & \multicolumn{1}{l}{MNAR} & BLAMM & \multicolumn{1}{l}{MNAR} & BLAMM & \multicolumn{1}{l}{MNAR} \bigstrut[b]\\
\hline
\hline
TARnet + mean & \multicolumn{1}{c}{\textbf{10.52$\pm$0.9}} & 0.02$\pm$1.6 & \textbf{8.49$\pm$1.7} & 0.36$\pm$1.4 & \textbf{7.26$\pm$2.3} & 0.04$\pm$1.8 & \textbf{1.62$\pm$2.0} & 1.57$\pm$0.2 \bigstrut[t]\\
Tnet + mean & \multicolumn{1}{c}{\textbf{10.42$\pm$0.9}} & -5.03$\pm$0.9 & \textbf{7.44$\pm$1.4} & -2.47$\pm$1.9 & \textbf{3.67$\pm$1.4} & -3.89$\pm$1.9 & \textbf{-1.5$\pm$0.1} & -4.9$\pm$1.2 \\
linear +mean & \multicolumn{1}{c}{\textbf{9.86$\pm$0.1}} & -1.21$\pm$0.1 & \textbf{10.1$\pm$0.0} & -1.26$\pm$0.1 & \textbf{6.29$\pm$0.0} & -0.91$\pm$0.1 & \textbf{2.26$\pm$0.0} & -1.0$\pm$0.0 \\
CF + mean & \multicolumn{1}{c}{\textbf{7.65$\pm$0.5}} & -1.41$\pm$0.1 & \textbf{3.32$\pm$0.1} & -1.52$\pm$0.0 & \textbf{3.01$\pm$0.0} & -1.25$\pm$0.2 & \textbf{1.54$\pm$0.0} & -1.18$\pm$0.1 \\
TARnet + MICE & \textbf{8.27$\pm$2.5} & -1.29$\pm$1.4 & \textbf{8.0$\pm$2.3} & 0.55$\pm$1.2 & \textbf{7.39$\pm$1.4} & 0.27$\pm$1.1 & \textbf{1.9$\pm$1.2} & -0.08$\pm$1.4 \\
Tnet +MICE & \textbf{6.08$\pm$1.6} & -3.03$\pm$1.9 & \textbf{5.84$\pm$1.0} & -2.84$\pm$1.6 & \textbf{2.93$\pm$2.4} & -2.27$\pm$1.0 & \textbf{-0.18$\pm$1.2} & -5.09$\pm$0.5 \\
linear +MICE & \textbf{6.87$\pm$1.3} & -1.44$\pm$0.1 & \textbf{7.24$\pm$0.5} & -1.6$\pm$0.1 & \textbf{5.2$\pm$0.3} & -1.43$\pm$0.2 & \textbf{1.39$\pm$0.2} & -1.25$\pm$0.1 \\
CF + MIA & \textbf{9.0$\pm$0.5} & -1.58$\pm$0.1 & \textbf{3.47$\pm$0.2} & -1.65$\pm$0.1 & \textbf{-0.41$\pm$0.1} & -1.37$\pm$0.1 & \textbf{-0.95$\pm$0.1} & -1.26$\pm$0.1 \\
CF (DR) + MIA & \textbf{1.59$\pm$0.3} & -1.5$\pm$0.1 & \textbf{0.61$\pm$0.2} & -1.58$\pm$0.1 & \textbf{8.79$\pm$0.1} & -1.58$\pm$0.2 & \textbf{2.79$\pm$0.1} & -1.41$\pm$0.1 \\
\end{tabular}%

	}
	\label{app:tab:mnar}%
\end{table}%

% Table generated by Excel2LaTeX from sheet 'edit'
%\usepackage{graphicx}
%\usepackage{multirow}
%\usepackage{booktabs}
%\usepackage{xcolor}
%\usepackage{bigstrut}

\begin{table}[h!]
	\centering
	\caption{BLAMM attack is successful while both untargeted missingness MCAR and MCAR are not. The average ATE (\%) of regression estimators (over 5 trials) under complete access. Bold highlights the missingness closer to the target ATE of 10\% and notice the ground truth ATE is -1.61\%.}
	\begin{tabular}{lcllll}
	& \multicolumn{1}{l}{Access 100\%} &       &       &       &  \\
	& \multicolumn{1}{l}{BLAMM (7.2\%)} & MCAR (7.2\%) & MNAR (7.2\%) & MNAR (25\%) & MNAR (50\%) \bigstrut[b]\\
	\hline
	\hline
	TARnet + mean & \textbf{10.52+-0.9} & -1.0+-2.2 & \textbf{0.02+-1.6} & \textbf{0.91+-1.4} & -0.34+-1.8 \bigstrut[t]\\
	Tnet + mean & \textbf{10.42+-0.9} & -2.6+-3.1 & -5.03+-0.9 & -4.02+-1.5 & -1.65+-3.0 \\
	linear +mean & \textbf{9.86+-0.1} & -1.45+-0.2 & \textbf{-1.21+-0.1} & \textbf{-1.23+-0.1} & \textbf{-0.82+-0.1} \\
	CF + mean & \textbf{7.65+-0.5} & -1.44+-0.2 & \textbf{-1.41+-0.1} & -1.64+-0.3 & -1.9+-0.2 \\
\end{tabular}%
% Table generated by Excel2LaTeX from sheet 'combined_mask_-1'

	\label{app:tab:mnar2}%
\end{table}%

%\begin{table}[h!]
%	\centering
%	\caption{The average ATE (\%) of regression estimators, $\hat{\tau}$ (over 5 trials). ``Access:'' indicates the percentage of rows manipulatable. Bold highlights the missingness closer to the target ATE of 10\%. \newline}
%		\input{table_ate_n_5_comb_mnar}
%\end{table}%
%

\end{document}